\documentclass[lettersize,journal]{IEEEtran}

\usepackage[utf8]{inputenc}
\usepackage[english]{babel}

\usepackage{amsmath,amsthm,amsfonts,amssymb,bm}
\usepackage{mathtools}

\usepackage{graphicx}
\usepackage{hyperref}
\hypersetup{colorlinks=false, hidelinks}
\usepackage{color}

\usepackage{booktabs}
\usepackage{multirow}
\usepackage{arydshln}

\usepackage{url}

\usepackage[numbers,square,sort&compress]{natbib}

\usepackage{tikz}

\newcommand{\cu}[1]{
	\ifcat\noexpand#1\relax
	\bm{#1}
	\else
	\mathbf{#1}
	\fi
}

\newcommand{\diff}{\mathop{}\!\mathrm{d}}

\newcommand{\expp}{\mathrm{e}}

\newcommand{\cond}{{\;|\;}}

\let\inf\relax
\let\lim\relax
\DeclareMathOperator*{\inf}{inf\,}  
\DeclareMathOperator*{\lim}{lim\,}  

\newcommand{\expecsym}{\operatorname{\mathbb{E}}}     
\newcommand{\covsym}{\operatorname{Cov}}     
\newcommand{\varrsym}{\operatorname{Var}}     
\newcommand{\diagsym}{\operatorname{diag}}     
\newcommand{\tracesym}{\operatorname{tr}}           

\let\expec\relax
\let\cov\relax
\let\varr\relax
\let\diag\relax
\let\trace\relax

\makeatletter
\newcommand{\expec}{\@ifstar{\@expecauto}{\@expecnoauto}}
\newcommand{\@expecauto}[1]{\expecsym \left[ #1 \right]}
\newcommand{\@expecnoauto}[1]{\expecsym [#1]}
\newcommand{\expecbig}[1]{\expecsym \bigl[ #1 \bigr]}
\newcommand{\expecBig}[1]{\expecsym \Bigl[ #1 \Bigr]}
\newcommand{\expecbigg}[1]{\expecsym \biggl[ #1 \biggr]}

\newcommand{\cov}{\@ifstar{\@covauto}{\@covnoauto}}
\newcommand{\@covauto}[1]{\covsym \left[ #1 \right]}
\newcommand{\@covnoauto}[1]{\covsym [#1]}
\newcommand{\covbig}[1]{\covsym \bigl[ #1 \bigr]}

\newcommand{\varr}{\@ifstar{\@varrauto}{\@varrnoauto}}
\newcommand{\@varrauto}[1]{\varrsym \left[ #1 \right]}
\newcommand{\@varrnoauto}[1]{\varrsym [#1]}

\newcommand{\diag}{\@ifstar{\@diagauto}{\@diagnoauto}}
\newcommand{\@diagauto}[1]{\diagsym \left( #1 \right)}
\newcommand{\@diagnoauto}[1]{\diagsym (#1)}

\newcommand{\trace}{\@ifstar{\@traceauto}{\@tracenoauto}}
\newcommand{\@traceauto}[1]{\tracesym \left( #1 \right)}
\newcommand{\@tracenoauto}[1]{\tracesym (#1)}

\makeatother

\newcommand*{\trans}{{\mkern-1.5mu\mathsf{T}}}

\newcommand*{\R}{\mathbb{R}} 

\let\norm\relax
\DeclarePairedDelimiter{\normbracket}{\lVert}{\rVert}
\newcommand{\norm}{\normbracket}

\let\abs\relax
\DeclarePairedDelimiter{\absbracket}{\lvert}{\rvert}
\newcommand{\abs}{\absbracket}
\newcommand{\absbig}[1]{\bigl \lvert #1 \bigr\rvert}

\newcommand{\absbigg}[1]{\biggl \lvert #1 \biggr\rvert}

\newcommand{\mBesselsec}{\operatorname{K_\nu}}

\def\matern{Mat\'{e}rn }

\makeatletter
\@ifundefined{thmenvcounter}{}
{%
	\newtheorem{envcounter}{EnvcounterDummy}[\thmenvcounter]
	\newtheorem{theorem}[envcounter]{Theorem}
	\newtheorem{proposition}[envcounter]{Proposition}
	\newtheorem{lemma}[envcounter]{Lemma}
	\newtheorem{corollary}[envcounter]{Corollary}
	\newtheorem{remark}[envcounter]{Remark}
	\newtheorem{example}[envcounter]{Example}
	\newtheorem{definition}[envcounter]{Definition}
	\newtheorem{algorithm}[envcounter]{Algorithm}
	\newtheorem{assumption}[envcounter]{Assumption}
}
\makeatother

\makeatletter
\def\adl@drawiv#1#2#3{%
	\hskip-.314159\tabcolsep
	\xleaders#3{#2.5\@tempdimb #1{1}#2.5\@tempdimb}%
	#2\z@ plus1fil minus1fil\relax
	\hskip-.314159\tabcolsep}
\newcommand{\cdashlinelr}[1]{%
	\noalign{\vskip\aboverulesep
		\global\let\@dashdrawstore\adl@draw
		\global\let\adl@draw\adl@drawiv}
	\cdashline{#1}
	\noalign{\global\let\adl@draw\@dashdrawstore
		\vskip\belowrulesep}}
\makeatother

\newtheorem{theorem}{Theorem}

\newtheorem{lemma}[theorem]{Lemma}
\newtheorem{corollary}[theorem]{Corollary}

\newtheorem{algorithm}[theorem]{Algorithm}
\newtheorem{assumption}[theorem]{Assumption}

\begin{document}

\title{Probabilistic Estimation of Instantaneous Frequencies of Chirp Signals}

\author{Zheng Zhao,~\IEEEmembership{Member,~IEEE},~Simo S\"{a}rkk\"{a},~\IEEEmembership{Senior Member,~IEEE}, Jens Sj\"{o}lund, \\and~Thomas B. Sch\"{o}n,~\IEEEmembership{Senior Member,~IEEE}%
\thanks{Zheng Zhao, Jens Sj\"{o}lund, and Thomas B. Sch\"{o}n are with Department of Information Technology, Uppsala University, Sweden (e-mail: \href{mailto:zheng.zhao@it.uu.se}{zheng.zhao@it.uu.se}). Simo S\"{a}rkk\"{a} is with Department of Electrical Engineering and Automation, Aalto University, Finland.}%
\thanks{This research was partially supported by the Wallenberg AI, Autonomous Systems and Software Program (WASP) funded by Knut and Alice Wallenberg Foundation and by Kjell och M\"{a}rta Beijer Foundation. The computations handling was enabled by resources provided by the Swedish National Infrastructure for Computing (SNIC) at NSC partially funded by the Swedish Research Council through grant agreement no. 2018-05973. }%
\thanks{Manuscript received ** **, ****; revised ** **, ****.}}

\markboth{Journal of \LaTeX\ Class Files,~Vol.~14, No.~8, August~2021}%
{Shell \MakeLowercase{\textit{et al.}}: A Sample Article Using IEEEtran.cls for IEEE Journals}

\maketitle

\begin{abstract}
	We present a continuous-time probabilistic approach for estimating the chirp signal and its instantaneous frequency function when the true forms of these functions are not accessible. Our model represents these functions by non-linearly cascaded Gaussian processes represented as non-linear stochastic differential equations. The posterior distribution of the functions is then estimated with stochastic filters and smoothers. We compute a (posterior) Cram\'{e}r--Rao lower bound for the Gaussian process model, and derive a theoretical upper bound for the estimation error in the mean squared sense. The experiments show that the proposed method outperforms a number of state-of-the-art methods on a synthetic data. We also show that the method works out-of-the-box for two real-world datasets.
\end{abstract}

\begin{IEEEkeywords}
	chirp signal, frequency estimation, frequency tracking, instantaneous frequency, state-space methods, Gaussian process, Kalman filtering, smoothing, automatic differentiation
\end{IEEEkeywords}

\section{Introduction}
\label{sec:intro}
Chirp signals are elementary and ubiquitous objects in signal processing. We consider a real-valued chirp signal model of the form
\begin{equation}
	\begin{split}
		r(t) &= \alpha(t)\sin\biggl(\phi_0 +  2 \,\pi\int^t_0  f(s) \diff s \biggr), \\
		Y_k &= r(t_k) + \xi_k, 
	\end{split}
	\label{equ:chirp-def}
\end{equation}
where $r\colon [0,\infty)\to\R$ is the chirp signal, $\alpha\colon [0,\infty)\to\R$ is the instantaneous amplitude, $\phi_0\in\R$ is the initial phase, and $f\colon [0,\infty)\to [0,\infty)$ is the \emph{instantaneous frequency (IF)} function of the signal $r$. We assume that for any time $t_k$, for time steps $k=1,2,\ldots,T$, we measure the chirp signal, denoted by $Y_k$, corrupted by an independent additive Gaussian noise $\xi_k\sim \mathrm{N}(0, \Xi)$. Estimating the IF from measurements of a chirp signal plays a crucial role in a number of real-world applications, such as radar tracking, communications, frequency modulation, and seismic attributes.

The goal of this paper is to develop a continuous-time and probabilistic approach for jointly estimating the IF function $f$ and the chirp signal $r$. Importantly, we do not assume any parametric forms of $f$ or the amplitude $\alpha$, and we aim at estimating the probability distribution of $f$ conditioned on the measurements. Not assuming a parametric form is desirable, since in reality, it is often difficult to know the explicit forms of the instantaneous frequency and amplitude functions. For instance, when estimating the Doppler frequency in radar tracking, we need to know the dynamics of the tracked object to model $f$, which is often impossible. As another example, the amplitudes of sound recordings usually have random perturbations which do not admit a particular parametric form~\cite{Doweck2015}. 

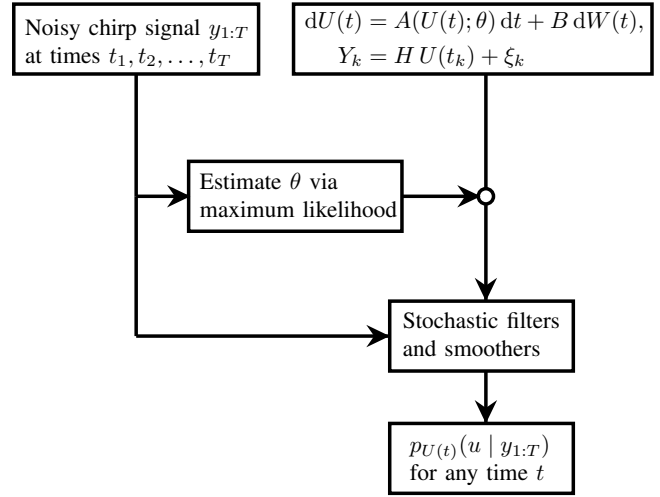
\begin{figure}[t!]
	\centering
	\resizebox{\linewidth}{!}{%
\tikzset{every picture/.style={line width=0.75pt}} 

\begin{tikzpicture}[x=0.75pt,y=0.75pt,yscale=-1,xscale=1]
	
	\draw  [line width=1.5]  (130,10) -- (270,10) -- (270,50) -- (130,50) -- cycle ;
	\draw  [line width=1.5]  (290,10) -- (495,10) -- (495,50) -- (290,50) -- cycle ;
	\draw  [line width=1.5]  (230,100) -- (352,100) -- (352,140) -- (230,140) -- cycle ;
	\draw  [line width=1.5]  (345,250) -- (450,250) -- (450,290) -- (345,290) -- cycle ;
	\draw  [line width=1.5]  (345,180) -- (450,180) -- (450,220) -- (345,220) -- cycle ;
	
	\draw [line width=1.5]    (200,50) -- (200,120) ;
	
	\draw [line width=1.5]    (200,120) -- (200,200) ;
	
	\draw [line width=1.5]    (200,200) -- (345,200) ;
	\draw [shift={(345,200)}, rotate = 180] [fill={rgb, 255:red, 0; green, 0; blue, 0 }  ][line width=0.08]  [draw opacity=0] (13.4,-6.43) -- (0,0) -- (13.4,6.44) -- (8.9,0) -- cycle    ;
	
	\draw [line width=1.5]    (200,120) -- (230,120) ;
	\draw [shift={(231,120)}, rotate = 180] [fill={rgb, 255:red, 0; green, 0; blue, 0 }  ][line width=0.08]  [draw opacity=0] (13.4,-6.43) -- (0,0) -- (13.4,6.44) -- (8.9,0) -- cycle    ;
	
	\draw [line width=1.5]    (400,220) -- (400,250) ;
	\draw [shift={(400,250)}, rotate = 270] [fill={rgb, 255:red, 0; green, 0; blue, 0 }  ][line width=0.08]  [draw opacity=0] (13.4,-6.43) -- (0,0) -- (13.4,6.44) -- (8.9,0) -- cycle    ;
	
	\draw [line width=1.5]    (352,120) -- (395,120) ;
	\draw [shift={(397,120)}, rotate = 180] [fill={rgb, 255:red, 0; green, 0; blue, 0 }  ][line width=0.08]  [draw opacity=0] (13.4,-6.43) -- (0,0) -- (13.4,6.44) -- (8.9,0) -- cycle    ;
	
	\draw [line width=1.5]    (400,50) -- (400,115) ;
	\draw [shift={(400,120)}, rotate = 90] [color={rgb, 255:red, 0; green, 0; blue, 0 }  ][line width=1.5]      (0, 0) circle [x radius= 4.36, y radius= 4.36]   ;
	
	\draw [line width=1.5]    (400,125) -- (400,180) ;
	\draw [shift={(400,180)}, rotate = 270] [fill={rgb, 255:red, 0; green, 0; blue, 0 }  ][line width=0.08]  [draw opacity=0] (13.4,-6.43) -- (0,0) -- (13.4,6.44) -- (8.9,0) -- cycle    ;
	
	\draw (135,17) node [anchor=north west][inner sep=0.75pt]   [align=left] {Noisy chirp signal $y_{1:T}$\\at times $t_1, t_2, \ldots, t_T$};
	
	\draw (295,11) node [anchor=north west][inner sep=0.75pt]    {$\begin{aligned}
			\diff U(t) &= A(U(t); \theta) \diff t + B \diff W(t), \\
			Y_k &= H \, U(t_k) + \xi_k
		\end{aligned}$};
	
	\draw (355,255) node [anchor=north west][inner sep=0.75pt]   [align=left] {$ p_{U(t)}( u \cond y_{1:T})$ \\for any time $ t$};
	
	\draw (351,185) node [anchor=north west][inner sep=0.75pt]   [align=left] {Stochastic filters \\and smoothers};
	
	\draw (235,105) node [anchor=north west][inner sep=0.75pt]   [align=left] {Estimate $ \theta $ via \\maximum likelihood};
\end{tikzpicture}
	}
	\caption{A flowchart of the proposed scheme. The state $U(t)$ encodes the chirp signal $H \, U(t)$ and its instantaneous frequency $f(t) = g(H_V \, U(t))$ jointly via a non-linear stochastic differential equation. The scope of this paper is to estimate the distribution of $U(t)$ based on the chirp measurements $y_{1:T}$. \looseness=-1}
	\label{fig:flowchart}
\end{figure}

\subsection{Contributions}
\label{sec:contributions}
The core technical contributions of the paper are the following.
\begin{itemize}
	\item We design a continuous-time probabilistic model to jointly characterise chirp signals and their instantaneous frequency functions.
	\item We propose a computationally efficient stochastic filtering and smoothing based method for probabilistic inference in this model.
	\item Additionally, we compute a (posterior) Cram\'{e}r--Rao lower bound for the model, and an upper bound for the mean-squared estimation error. 
\end{itemize}
The proposed instantaneous frequency estimation scheme is illustrated in Figure~\ref{fig:flowchart}. It is composed of two non-linearly cascaded Gaussian processes (GPs) which are represented as a non-linear system of stochastic differential equations (SDEs). The model simultaneously enables an expressive paradigm and efficient inference techniques as we view the IF estimation problem as a non-parametric GP-SDE regression problem. The rationale and upsides of using such a GP-SDE model for the IF identification are as follows:
\begin{itemize}
	\item The GP-SDE model covers a wide class of chirp and IF functions in terms of probability distributions. This implies that we do \emph{not} have to postulate strong assumptions on the functions, such as the commonly used piecewise linearity in short-time windows. The probabilistic model also allows us to quantify the uncertainty in the estimates.
	\item The GP-SDE model is flexible. The users can choose a suitable GP model depending on their applications. For instance, if we know that the true IF function is a polynomial, then, we can choose a GP kernel that generates polynomial samples~\cite{Rasmussen2006}.
	\item The GP-SDE model is in continuous-time. We can query the chirp and IF estimates at any given time (i.e., do interpolation or extrapolation).
	\item The (non-linear) SDE representation of the GP model is computationally efficient. By using stochastic filters and smoothers, we can estimate the posterior distribution of the chirp and IF with a computational complexity that is linear in the number of measurements.
	\item The parameters of the GP-SDE model can be efficiently estimated using maximum likelihood and modern automatic differentiation techniques.
\end{itemize}
To show that the proposed scheme is theoretically sound, we compute a (posterior) Cram\'{e}r--Rao lower bound and derive a mean-square upper bound for the model. The numerical experiments also show that the proposed method outperforms a number of baseline and state-of-the-art methods in terms of the estimation error. Lastly, we apply the method to compute the IFs of a gravitational wave and two European bats calls. These real-data experiments show that the method works with real-world data out-of-the-box. Our implementation of the method is publicly accessible\footnote{\url{https://github.com/spdes/chirpgp}.}.

\subsection{Related work}
\label{sec:related-works}
In this section, we review a number of existing IF estimation methods for~\eqref{equ:chirp-def}, and briefly explain how they relate to the proposed method. The simplest IF estimation scheme is arguably to use the Hilbert transform~\cite{Boashash1992}. This method first finds the instantaneous phase of the analytic representation of the signal (which is the Hilbert transform of the signal), and then computes the IF by differentiating the unwrapped instantaneous phase. However, in reality this method only works well for clean signals. 

Practitioners often use the time-windowed power spectral density method to estimate the IF~\cite[][Eq. 22]{Boashash1992}. The idea is to first estimate the power spectral densities in time windows, and then estimate the IF in each window from the first moment of the estimated power spectral density. However, the method requires a lot of manual tuning of parameters (e.g., window type, length, and overlap), and the estimation in the time-frequency domain is restricted by Gabor's uncertainty principle. This method also needs interpolation to deal with unevenly sampled measurements. In comparison, our model supports automatic parameter tuning via maximum likelihood and can be applied, unchanged, to unevenly sampled measurements, since the model is formulated in continuous time. 

There also exists a plethora of parametrisation-based methods for estimating the IF. These methods postulate a parametrised form of the IF, and then estimate its parameters using, for example, maximum likelihood~\cite{Djuric1990}. The most commonly used parametrisation is that of an affine function which results in a chirp signal with quadratic phase. For these models, it is possible to develop asymptotically exact and computationally efficient methods to estimate the parameters. As an example, reference \cite{Jensen2017} exploits the matrix structure of the non-linear least square estimator of the parameters to achieve a sub-cubic computational complexity in the number of samples. 

However, in reality, chirp signals do not always have a linear IF. To tackle more general IFs, we can heuristically define a sequence of short-time windows, and then approximate the IF by a piecewise linear function across the windows. The upside of this approach is that we can directly make use of the refined linear IF estimators, such as the ones in \cite{Jensen2017} and~\cite{Nielsen2017}. The downsides are that the IF is assumed to change slowly in time, and that the performance is limited by the uncertainty principle: for the maximum likelihood estimator to be accurate we need to use large enough time-windows, while for the piecewise linearity assumption to succeed, we need to use small enough time-windows~\cite{Norholm2016}. In addition to the window-based methods, it is also customary to use non-linear parametrisations than the linear ones, such as polynomials~\cite{Djuric1990}. However, this comes at the cost of more computations and over/mis-specification. For instance, in radar tracking and audio processing, it is challenging to find a precise parametrisation of the underlying IF. There are also many known issues of the polynomial maximum likelihood estimators, such as numerical instability, optima oscillation, aliasing, and phase unwrapping~\cite{Farquharson2006, Djurovic2018}.

Apart from the parametric approaches, it is also common to put a prior (e.g., a Gaussian process) on the underlying IF and then make use of Bayesian estimators to find the posterior distribution of the IF. This relaxes the ambiguities caused by the aforementioned parametric approaches to a certain extent, since the prior can provide a richer family of functions than deterministic parametrisations. Moreover, since signals are temporal functions, it makes sense to consider Markov priors (e.g., state-space models) for the sake of efficient computations. An early work using this idea dates back to~\cite[][pp. 99]{Snyder1968} and it was later followed up by~\cite{Scala1996}. Essentially, this class of approaches and their variants (in the discrete-time representation) use a state-space model of the form
\begin{equation}
	\begin{split}
		f(t_k) &= \eta_{k-1} \, f(t_{k-1}) + \zeta(t_{k-1}), \\
		Y_k &= \alpha \sin\biggl( \phi_0 + 2 \, \pi \sum_{i=0}^k f(t_i)  \biggr) + \xi_k,
		\label{equ:old-ekf-ss}
	\end{split}
\end{equation}
for $k=1,2,\ldots$, where $\lbrace \eta_k \rbrace_{k\geq0}$ and $\lbrace \zeta_k \rbrace_{k\geq0}$ are some scalar coefficients and Gaussian random variables, respectively, that describe the transition dynamics of the IF. However, the model in~\eqref{equ:old-ekf-ss} assumes that the chirp amplitude is constant which usually does not hold for real-world data. A closely related but different state-space model is given in~\cite{Nielsen2011},~\cite{Shi2017}, and~\cite{Shi2019}. The model additionally encodes the instantaneous amplitude and phase in its state vector which end up with a non-linear measurement model. However, the sampling-based inference scheme in~\cite{Nielsen2011} is computationally expensive, and the dynamic models in~\cite{Shi2017} and~\cite{Shi2019} do not exploit the prior information (e.g., continuity and volatility) of the underlying IF function. Furthermore, their state-space models are in discrete-time, which implicitly assumes that the IF is locally linear. 

In this paper, we employ the observation that we can simplify~\eqref{equ:old-ekf-ss} by replacing the sine function and the cumulative sum of $f$ in~\eqref{equ:old-ekf-ss} with a harmonic equation representation~\cite[see, e.g.,][pp. 2]{Scala1995}. This results in a state-space model with non-linear dynamics and linear measurements implying that we can use conventional non-linear filters, such as the extended Kalman filters (EKFs)~\cite{Scala1996} and sigma-point filters~\cite{Dardanelli2010} to approximate the posterior distribution of $f$. We derive this type of state-space model based on the deep connection between SDEs and GP models, so that we can encode the prior knowledge of the IF by selecting an appropriate GP kernel. The resulting model is a non-linear SDE that jointly correlates chirp signals and their IF functions in continuous time. In a modern interpretation, our design corresponds to a hierarchical statistical model that is a special case of the deep Gaussian processes~\cite[][pp. 82]{Zhao2021Thesis}.

\subsection{Outline of the paper}
The paper is organised as follows. In Section~\ref{sec:if-as-gp}, we show how to view the IF estimation problem as a GP-SDE regression problem. In the same section, we show how to design the GP-SDE prior appropriately so as to jointly model chirp signals and their IFs. We then show how to solve the estimation problem by using stochastic filters and smoothers. The error analysis and numerical experiments are presented in Sections~\ref{sec:error-analysis} and~\ref{sec:experiments}, respectively, followed by the conclusions in Section~\ref{sec:conclusion}.

\section{Representing chirp signal and its instantaneous frequency using a Gaussian process}
\label{sec:if-as-gp}

Gaussian processes (GPs) are ubiquitous models used, for example, in statistics and machine learning communities for modelling unknown latent functions~\cite{Rasmussen2006}. To model the chirp signal and its underlying IF, we use the following GPs. Let $X\colon [0, \infty)\to\R^{d_x}$ be a (non-stationary) GP defined by
\begin{equation}
	X(t) \sim \mathrm{GP}\big(m_X(t; f), C_X(t, t'; f)\big)
	\label{equ:gp-X}
\end{equation}
for some mean function $m_X(t; f) = \expec{X(t) \cond f}$ and covariance function $C_X(t, t'; f) = \covbig{(X(t) - m_X(t)) \, (X(t') - m_X(t'))^\trans \cond f}$, parametrised by a function $f$ which acts as the associated IF. Since the chirp signal is $\R$-valued, we use $t \mapsto H_X \, X(t)$ to represent the chirp signal under some linear transformation $H_X\colon \R^{d_x}\to\R$. As for the IF $f$, we model it to be driven by a GP as well, according to
\begin{equation}
	\begin{split}
		f(t) &\coloneqq g( V(t)),\\
		V(t) &\sim \mathrm{GP}\big(0, C_V(t, t')\big),
	\end{split}
	\label{equ:prior-V}
\end{equation}
where $V$ is a zero-mean GP that completely characterises $f$ in conjunction with a positive bijection $g\colon\R\to\R_{>0}$, for instance, $g(\cdot) = \exp(\cdot)$ or $g(\cdot) = \log(1 + \exp(\cdot))$. The reason for introducing $g$ is to ensure that the IF is positive, otherwise the estimator may not correctly identify the sign of the IF. Due to this, $f$ is technically not a GP, but its driving term $V$ is. Unless otherwise stated, we refer $f$ as a GP for simplicity. The bijection $g$ can be used for other purposes as well, for example, to introduce the carrier frequency in frequency modulations by adding a constant offset in $g$.

We let the random variable $Y_k \in \R$ represent the noisy measurement of the chirp signal at time $t_k$:
\begin{equation*}
	Y_k = H_X \, X(t_k) + \xi_k,
\end{equation*}
where $H_X\colon \R^{d_x}\to\R$ is the linear transformation that extracts the chirp from $X$, and $\xi_k \sim \mathrm{N}(0, \Xi)$. Suppose that we have a set of measurements $y_{1:T}\coloneqq \lbrace y_k \rbrace_{k=1}^T$ at hand. The goal is then to compute the joint posterior density
\begin{equation}
	p_{X(t), V(t)}(x, v \cond y_{1:T})
	\label{equ:posterior-goal}
\end{equation}
of $X(t)$ and $V(t)$ marginally for all $t\in[0, \infty)$.

In practice, solving the estimation problem in~\eqref{equ:posterior-goal} is hard. One difficulty lies in how to find a concrete pair of the mean $m_X$ and covariance $C_X$ functions such that $X$ is a reasonable prior for chirp signals. By ``reasonable'', we mean that the samples drawn from $X$ should be valid chirp signals of the form in~\eqref{equ:chirp-def}, and that the IFs of the samples should follow~\eqref{equ:prior-V}. The other difficulties consist in the intractability and expensive computation of~\eqref{equ:posterior-goal}, as the joint prior model of $X$ and $f$ is not a conventional GP but a hierarchical deep GP~\cite{Zhao2020SSDGP}.

In what follows, we give solutions to tackle these difficulties. Specifically, in Section~\ref{sec:chirp-prior}, we construct $X$ via a class of harmonic non-stationary SDE-GPs, and in Section~\ref{sec:model-formulation}, we leverage this type of SDE-GPs to formulate a continuous-discrete state-space model that we use to characterise the chirp signals. In Section~\ref{sec:filtering-smoothing}, we exemplify Gaussian filters and smoothers, and their log-likelihood for estimating the posterior distribution~\eqref{equ:posterior-goal} and the model parameters. Lastly, in Section~\ref{sec:harmonic-chirp} we show how to extend the model to tackle chirp signals containing several harmonic frequencies.

\subsection{Chirp signal prior}
\label{sec:chirp-prior}
In light of the periodic nature of chirp signals, it is natural to use the periodic covariance function in \cite[][pp. 92]{Rasmussen2006} for $C_X$ to construct the prior GP $X$ in~\eqref{equ:gp-X}. However, this periodic covariance function can only deal with constant frequencies, as it may fail to be positive definite when its frequency parameter is time-dependent. Moreover, even after we have found a valid and meaningful $C_X$, the computations required to obtain the posterior distribution~\eqref{equ:posterior-goal} are demanding because the inversion of the covariance matrix has a cubic complexity in the number of measurements. Indeed there are sparse pseudo-input methods to approximate the full-rank GP covariance matrix~\cite[e.g.,][]{Snelson2006}, however, one must be cautious in using them for IF estimation, since these sparse approximations implicitly introduce down-sampling. Moreover, the down-sampling rates are not easy to control, as the pseudo-inputs (i.e., the decimation points) are placed by an optimisation procedure~\cite{Snelson2006}. In order to approximate the posterior density~\eqref{equ:posterior-goal}, we often have to use, for instance, variational Bayes or Markov chain Monte Carlo methods which can be computationally expensive as well. These matters make it difficult to deal with long signals, such as audio recordings.

For the aforementioned reasons, we choose to construct $X$ via linear stochastic differential equations (SDEs), the solutions of which are (Markov) GPs with \emph{implicitly} defined $C_X$~\cite{Sarkka2019}. This relieves us from \emph{explicitly} designing a valid covariance function $C_X$ and from computing the covariance matrix inversion. As a result, we can compute the posterior density~\eqref{equ:posterior-goal} marginally in time without using the full covariance matrix. \looseness=-1

To see how the GP-SDE is formulated, let us first consider a simple class of harmonic GPs~\cite{Solin2014} that model sinusoidal signals with a \emph{constant} frequency $f$. These harmonic GPs are governed by a family of linear time-invariant SDEs of the form
\begin{equation}
	\begin{split}
		\diff X(t) &=
		\begin{bmatrix}
			-\lambda      & -2 \, \pi \, f \\
			2 \, \pi \, f & -\lambda
		\end{bmatrix}
		X(t) \diff t + b \diff W_X(t),\\
		X(0) &\sim \mathrm{N}\big(m^X_0, P^X_0\big),
		\label{equ:harmonic-sde}
	\end{split}
\end{equation}
where $X\colon[0,\infty)\to\R^2$ stands for the state, $W_X\colon [0,\infty)\to\R^2$ is a standard Wiener process, and $b\in\R$ is the dispersion coefficient. The parameter $\lambda$ is a positive damping constant representing the loss of energy which occurs in a number of real signals, such as seismic waves and elastic dynamics. This SDE has an analytical solution~\cite{Karatzas1991}, specifically, starting at any initial time $s\in[0,\infty)$, the solution at any $t\geq s$ is
\begin{equation}
	X(t) = F(t,s) \, X(s) + Q(t,s),
	\label{equ:harmonic-sde-solution}
\end{equation}
where
\begin{equation*}
	\begin{split}
		F(t,s) &= \exp\biggl( (t-s) 
		\begin{bmatrix}
			-\lambda      & -2 \, \pi \, f \\
			2 \, \pi \, f & -\lambda
		\end{bmatrix} \biggr) \\
		&=
		\begin{bmatrix}
			\cos((t-s) \, 2 \, \pi \,f)  & -\sin((t-s) \, 2 \, \pi \, f) \\
			\sin((t-s) \, 2 \, \pi \, f) & \cos((t-s) \, 2 \, \pi \, f)
		\end{bmatrix}\!\expp^{-\lambda (t-s)}, \\
		Q(t,s) &= b\int^t_s F(t, z) \diff W_X(z).
	\end{split}
\end{equation*}
The marginal distribution of $Q$ is Gaussian, that is, $Q(t,s) \sim \mathrm{N}(0, \Sigma(t,s))$, and its covariance is given by
\begin{equation}
	\begin{split}
		\Sigma(t,s) &\coloneqq \cov{Q(t, s)} \\
		&=b^2\int^t_s F(t, z) \, F(t, z)^\trans \diff z\\
		&= \begin{cases}
			b^2 \, \big(1 - \exp(-2\,\lambda \,(t-s))\big) \, / \, (2\,\lambda) \, I_2, & \lambda\neq 0,      \\
			b^2 \, (t-s) \, I_2,                                                        & \lambda=0,
		\end{cases}
		\nonumber
	\end{split}
\end{equation}
where $I_2\in\R^{2\times 2}$ stands for the identity matrix. In the remainder of the paper, we also interchangeably use the notation $\Sigma(\Delta)$, since this $\Sigma$ only depends on the time difference $\Delta=t-s$.

The solution process $X$ in~\eqref{equ:harmonic-sde-solution} is a suitable prior for modelling sinusoidal signals in the sense that its statistics (i.e., the mean and covariance) have a periodic structure. It is well-known~\cite{Karatzas1991} that the mean and covariance functions of $X$ are
\begin{equation*}
	m_X(t) \coloneqq \expec{X(t)} = F(t, 0) \, m^X_0
\end{equation*}
and
\begin{equation}
	\begin{split}
		C_X(t, t') &\coloneqq \cov{X(t), X(t')} \\
		&=
		\begin{cases}
			\cov{X(t)} \, F(t', t)^\trans, & t < t',    \\
			F(t, t') \cov{X(t')},          & t \geq t',
		\end{cases}\\
		\cov{X(t)} &= F(t, 0) \, P^X_0 \, F(t, 0)^\trans + \Sigma(t, 0),
	\end{split}
	\label{equ:cov-harmonic-sde}
\end{equation}
respectively. From these, we see that the mean
\begin{equation*}
	\expecBig{%
		\begin{bmatrix}
			0 & 1
		\end{bmatrix} X(t)
	} = \alpha \, \expp^{-\lambda \, t}\sin(\phi_0 + 2 \, \pi \, f \, t)
\end{equation*}
is a damped sinusoidal signal with the frequency parameter $f$, amplitude $\alpha = \norm{m^X_0}_2$, and initial phase $\phi_0 = \arctan(m^X_{0, 2}\, / \, m^X_{0, 1})$, where $m^X_{0, 1}$ and $m^X_{0, 2}$ stand for the first and second components of $m^X_0$, respectively. The covariance function $C_X$ is illustrated in Figure~\ref{fig:cov-harmonic-sde} for the frequency $f=0.5$~Hz. The figure clearly shows that $C_X$ has a periodic structure determined by $f$.

\begin{figure}[t!]
	\centering
	\includegraphics[width=.99\linewidth]{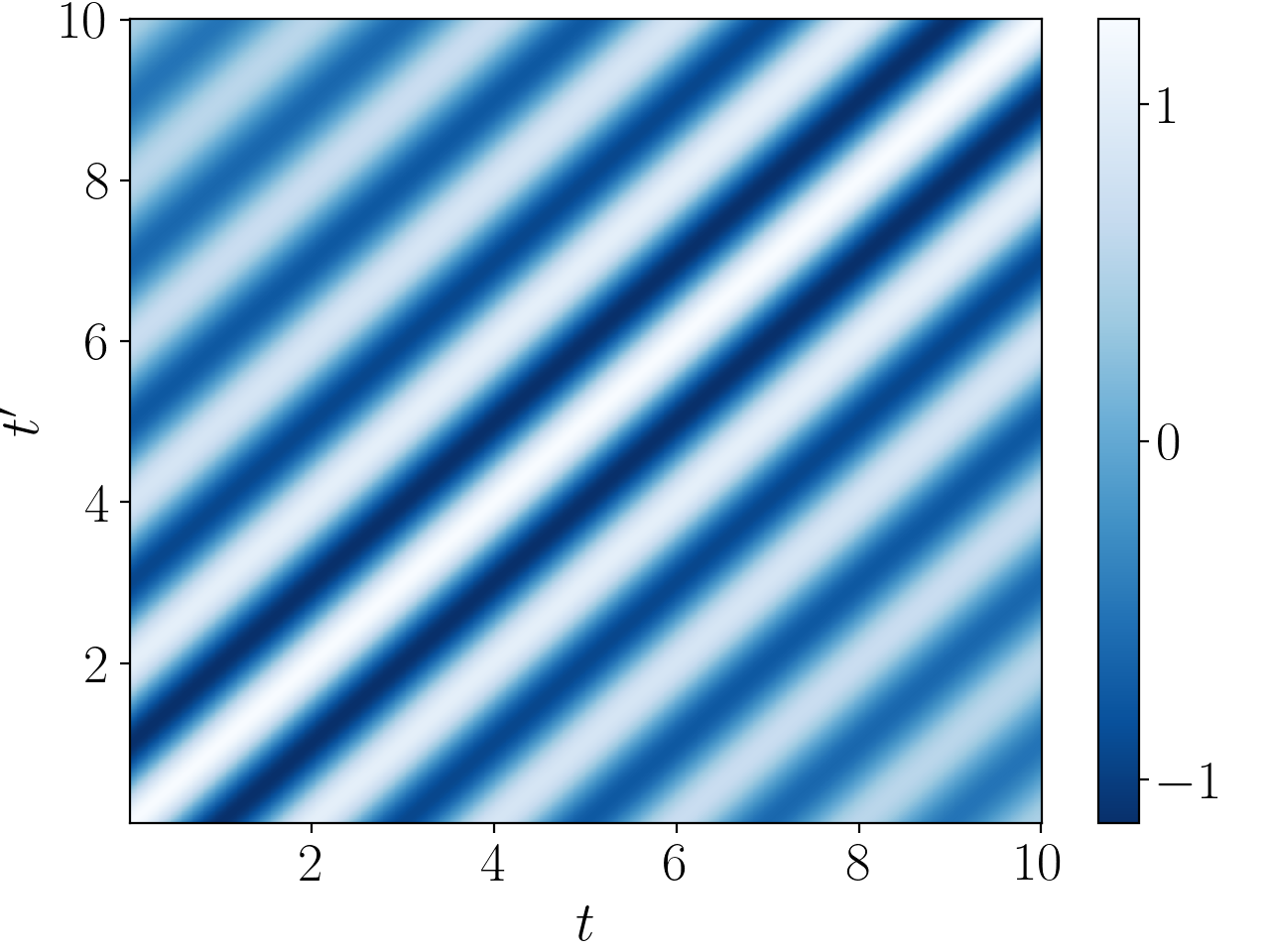}
	\caption{The covariance function $C_X$ in~\eqref{equ:cov-harmonic-sde} evaluated at Cartesian grid $[0, 10]\times [0, 10]$. The parameters used are $f=0.5$~Hz, $\lambda=0.1$, $b=0.5$, and $P^X_0=1.25 \, I_2$. The periodic structure is clearly visible, and so is the fading effect on the anti-diagonal (from top-left to bottom-right) due to the damping factor. \looseness=-1}
	\label{fig:cov-harmonic-sde}
\end{figure}

\subsection{GP-SDE for chirp signals and instantaneous frequencies}
\label{sec:model-formulation}
The SDE construction of $X$ in~\eqref{equ:harmonic-sde} allows us to use time-dependent instantaneous frequency by directly substituting its constant frequency parameter $f$ with a time-varying one, for instance, the GP $t\mapsto f(t)$ defined by~\eqref{equ:prior-V}. In this case, the conditional mean of $X$ becomes $\expec{X(t) \cond f(t)} = \alpha \, \expp^{-\lambda \, t} \begin{bmatrix} \cos(\phi_0 + 2 \, \pi \int^t_0 f(s) \diff s) & \sin(\phi_0 + 2 \, \pi \int^t_0 f(s) \diff s)\end{bmatrix}$ which has the model in~\eqref{equ:chirp-def} as a special case. As for the conditional covariance, it is not available in closed-form. However, we will show a numerically computed example of the conditional covariance in Figure~\ref{fig:cov-cond-f}.

Let us now consider the construction $f(t) = g(V(t))$ as defined in~\eqref{equ:prior-V}. In order to be consistent with the SDE construction of $X$, we need to represent $V$ via an SDE as well. Accordingly, we construct $V$ from an SDE representation $\overline{V}\colon [0, \infty) \to \R^{d_v}$ governed by
\begin{equation}
	\begin{split}
		\diff
		\overline{V}(t)
		&=
		M \, \overline{V}(t) \diff t + L \diff W_V(t), \\
		\overline{V}(0) &\sim \mathrm{N}\big(0, P^V_0\big),
		\label{equ:sde-of-V}
	\end{split}
\end{equation}
where $W_V\colon [0, \infty)\to\R^{d_w}$ is another standard Wiener process that is independent of $W_X$. We extract $V$ from the state $\overline{V}$ by $V(t) = \overline{H}_V \, \overline{V}(t)$ for some linear transformation $\overline{H}_V\colon \R^{d_v} \to \R$. \looseness=-1

The SDE coefficients $M\in\R^{d_v \times d_v}$ and $L\in\R^{d_v \times d_w}$ need to be chosen appropriately to model the underlying IF at hand. One useful choice is to let $V$ be a Mat\'{e}rn GP. The \matern GPs are generic priors for modelling continuous functions with varying degree of regularity, and their SDE representations are available in closed form~\cite{Sarkka2019}. As an example, suppose that the latent IF is continuously differentiable, then we can let
\begin{align}
	M &=
	\begin{bmatrix}
		0                & 1                           \\
		-3 \, /\, \ell^2 & -2 \, \sqrt{3} \, / \, \ell
	\end{bmatrix}, \quad
	L =
	\begin{bmatrix}
		0 \\ 2 \, \sigma \,(\sqrt{3} \, / \, \ell)^{3 \, / \, 2}
	\end{bmatrix},\nonumber\\
	P^V_0 &=
	\begin{bmatrix}
		\sigma^2 & 0                            \\
		0        & 3 \, \sigma^2 \, / \, \ell^2
	\end{bmatrix}, \qquad\,\,\,
	\overline{H}_V =
	\begin{bmatrix}
		1 & 0
	\end{bmatrix},
	\label{equ:sde-matern32}
\end{align}
and $\overline{V}(t) = \begin{bmatrix} V(t) & \diff V(t) \, / \diff t\end{bmatrix}^\trans \in\R^2$, so that $V$ is a Mat\'{e}rn ($\nu=3 \, / \, 2$) GP with the covariance function
\begin{equation*}
	\begin{split}
		C_V(t,t') &= \frac{\sigma^2 \, 2^{1-\nu}}{\Gamma(\nu)} \,\psi(t,t')^\nu \, \mathrm{K}_\nu\bigl(\psi(t,t')\bigr), \\
		\psi(t,t') &\coloneqq \frac{\sqrt{2 \, \nu} \, \abs{t-t'}}{\ell},
	\end{split}
\end{equation*}
where $\ell$ and $\sigma$ are the length and magnitude scale parameters (i.e., they determine the horizontal and vertical degrees of change of $V$), $\nu$ defines the smoothness of $V$, and $\mBesselsec$ is the modified Bessel function of the second kind.

For simplicity of exposition, we will keep using the Mat\'{e}rn $3\, / \, 2$ setting as in~\eqref{equ:sde-matern32} in the remainder of the paper. However, it is straightforward to use other classes of GPs as well by changing the SDE coefficients in~\eqref{equ:sde-of-V} accordingly. For example, if in an application we know that the IF is a rational function, then we can use the SDE representation of a rational quadratic GP to design the SDE coefficients.

\begin{figure}[t!]
	\centering
	\includegraphics[width=.99\linewidth]{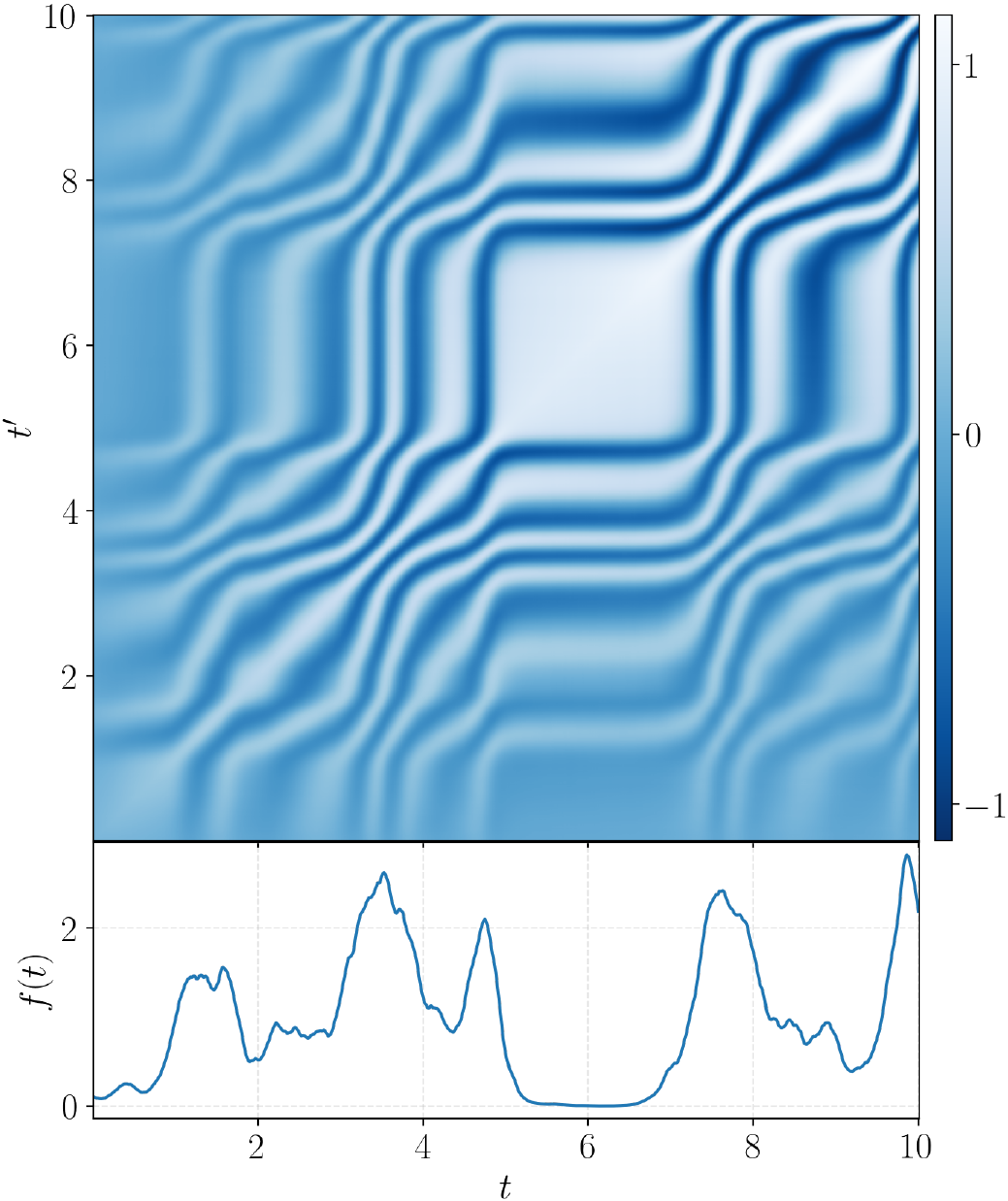}
	\caption{Top: the covariance function of $H \, U$ in the SDE~\eqref{equ:sde-chirp} conditioned on a realisation of $f$. Bottom: the realisation of $f$. From the plot we can see that the periodicity changes over time driven by the value of $f$. Note that this covariance function has been approximately computed using Monte Carlo, since it is not analytically tractable.}
	\label{fig:cov-cond-f}
\end{figure}

\begin{figure*}[t!]
	\centering
	\includegraphics[width=.99\linewidth]{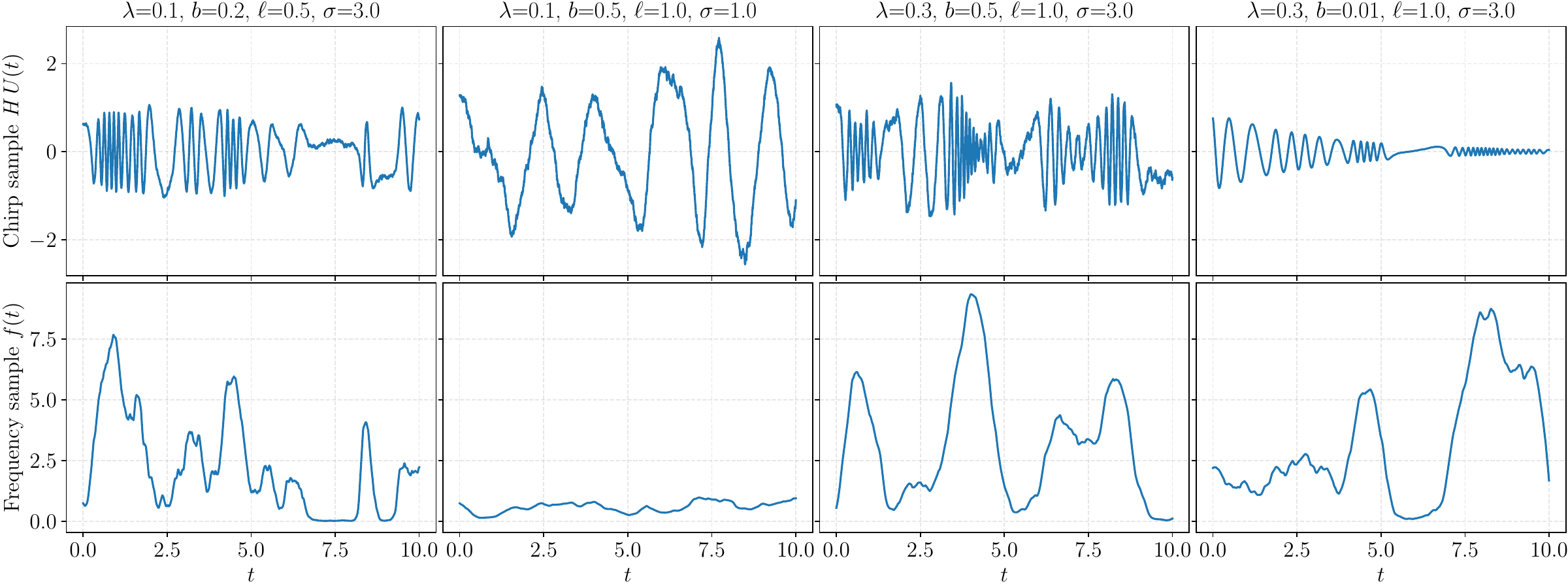}
	\caption{Samples drawn from the SDE in~\eqref{equ:sde-chirp} with four combinations of the model parameters. We see that the chirp frequency changes instantaneously based on the value of $f$, and that the model can be used to generate a rich variety of randomised chirp signals by tuning the model parameters. In this example, we choose the bijection in the way that $f(t) = \log(1 + \exp(V(t)))$.}
	\label{fig:chirp-sde-samples}
\end{figure*}

Now let us substitute the parameter $f$ in~\eqref{equ:harmonic-sde} with $f(t)=g(V(t))$, and define $U(t) \coloneqq \begin{bmatrix} X(t)^\trans & \overline{V}(t)^\trans\end{bmatrix}^\trans \in\R^4$ as the joint state by stacking $X(t) \in\R^2$ and $\overline{V}(t)\in\R^2$. We get the following non-linear state-space chirp IF estimation model:
\begin{equation}
	\begin{split}
		\diff U(t) &=
		A\big(U(t)\big)\diff t + B \diff W(t),\\
		U(0) &\sim p_{U(0)}(u), \\
		Y_k &= H \, U(t_k) + \xi_k,
	\end{split}
	\label{equ:sde-chirp}
\end{equation}
where the drift and dispersion functions are defined by
\begin{equation*}
	\begin{split}
		&A(U(t)) \\
		&\coloneqq
		\begin{bmatrix}
			-\lambda            & -2 \, \pi \, g(V(t)) & 0                & 0                           \\
			2 \, \pi \, g(V(t)) & -\lambda             & 0                & 0                           \\
			0                   & 0                    & 0                & 1                           \\
			0                   & 0                    & -3 \, /\, \ell^2 & -2 \, \sqrt{3} \, / \, \ell
		\end{bmatrix} U(t)
	\end{split}
\end{equation*}
and
\begin{equation*}
	B \coloneqq
	\begin{bmatrix}
		b & 0 & 0 & 0\\
		0 & b & 0 & 0\\
		0 & 0 & 0 & 0\\
		0 & 0 & 0 & 2 \, \sigma \,(\sqrt{3} \, / \, \ell)^{3 \, / \, 2}
	\end{bmatrix},
\end{equation*}
respectively, and the Wiener process $W(t)\in\R^4$ stacks the independent processes $W_X$ and $W_V$. The measurement operator $H\coloneqq\begin{bmatrix} 0 & 1 & 0 & 0 \end{bmatrix}$ extracts the second component of $X$ from $U$ so as to produce chirp signals of the form in~\eqref{equ:chirp-def}. The initial distribution can be assigned as a Normal $p_{U(0)}(u) = \mathrm{N}(u \cond m_0, P_0)$, where $m_0$ and $P_0$ are given by the initial means and covariances of $X$'s and $V$'s.

To see that the SDE in~\eqref{equ:sde-chirp} is an appropriate prior for modelling chirp signals with time-dependent IF, we can investigate the statistics of the SDE. In the beginning of this section, we have already shown that the conditional mean $\expec{H \, U(t) \cond f(t)}$ is a valid chirp signal of the form in~\eqref{equ:chirp-def}. Figure~\ref{fig:cov-cond-f} then shows the covariance function of $H \, U$ conditioned on a realisation of $f$. Compared to the covariance function $C_X$ with a fixed frequency plotted in Figure~\ref{fig:cov-harmonic-sde}, we see that the (white) stripes in Figure~\ref{fig:cov-cond-f} wiggle in time based on the value of $f$. This reflects the change of instantaneous periodicity driven by the process $f$. \looseness=-1

Another way to perceive the statistics of the SDE in~\eqref{equ:sde-chirp} is by inspecting its samples. It is worth noting that the commonly used Euler--Maruyama scheme is unstable for simulating this SDE due to the non-linearity and the stiffness of~\eqref{equ:harmonic-sde}. However, we can instead use higher-order methods, such as It\^{o}--Taylor expansions~\cite{ZhaoTME2020}, or the locally conditional discretisation (LCD)~\cite[][pp. 77]{Zhao2021Thesis}. The LCD method is particularly efficient for our SDE because it exploits the hierarchical structure of the SDE. In Figure~\ref{fig:chirp-sde-samples}, we exemplify a few samples drawn from the SDE~\eqref{equ:sde-chirp} by using the LCD method under four combinations of the model parameters (i.e., $\lambda, b, \ell$, and $\sigma$). From this figure it is clear that we can generate a rich variety of chirp signals and IFs by tuning the model parameters. The parameters $\lambda$ and $b$ control the damping factor and stochastic volatility of the chirp, while $\ell$ and $\sigma$ control the characteristics of the IF. For best performance, it is desirable to estimate these parameters from the chirp measurements to find the best fit to the particular chirp at hand.

\subsection{Estimate the posterior distribution and model parameters}
\label{sec:filtering-smoothing}
Computing the posterior probability density
\begin{equation}
	p_{U(t)}(u \cond y_{1:T}), \quad \text{for all } t \in[0,\infty),
	\label{equ:posterior-goal-2}
\end{equation}
for the state-space model in~\eqref{equ:sde-chirp} is equivalent to solving a (continuous-discrete) stochastic filtering and smoothing problem~\cite{Sarkka2019} on that model. The density can be computed in both continuous-time and discrete-time, but the continuous-time solution requires solving the Kushner partial differential equation. Although this posterior distribution is intractable also in the discrete-time case, there exists a number of approximate schemes, such as Gaussian filters and smoothers~\cite{Sarkka2019}, and particle filters and smoothers~\cite{Chopin2020}. Furthermore, estimating the model parameters can be accomplished by optimising the marginal log-likelihood computed by the filters.

At this stage, it is worth mentioning that this framework for estimating chirp IF is similar to that of~\cite{Scala1996}, except that our model is continuous in time, and that we formulate the model by using the GP tools that characterise a richer family of chirp signals and IFs. More specifically, the model in~\cite{Scala1996} is a special case of our SDE~\eqref{equ:sde-chirp} when choosing $b=0$ and when discretised using the LCD discretisation.

For the sake of pedagogy and the error analysis in Section~\ref{sec:error-analysis}, we show a concrete solution to the posterior distribution~\eqref{equ:posterior-goal-2} using the generic Gaussian filters and smoothers (GFSs)~\cite{Sarkka2013}. We choose GFSs for the demonstration because they generalise the extended and sigma-points Kalman filters and smoothers which are widely used in the signal processing community. These are shown in the following algorithm. 

\begin{algorithm}[Gaussian filtering and smoothing for the chirp and IF estimation model in~\eqref{equ:sde-chirp}]
	\label{alg:gaussian-filter-smoother}
	Find a Gaussian approximation to the SDE solution in~\eqref{equ:sde-chirp} in discrete-time such that
	\begin{equation*}
		\begin{split}
			U(t_k) &\approx \Phi(U(t_{k-1})) + \omega(U(t_{k-1})), \\
			\omega(U(t_{k-1})) &\sim \mathrm{N}(0, \Omega(U(t_{k-1}))),
		\end{split}
	\end{equation*}
	where $\Phi(u_{k-1}) \coloneqq \expec{U(t_k) \cond u_{k-1}}$ and $\Omega(u_{k-1}) \coloneqq \cov{U(t_k) \cond u_{k-1}}$ stand for the conditional mean and covariance of the SDE solution $U$, respectively. See, for example, \cite{Sarkka2019} for how to do so.

	Gaussian filters and smoothers approximate $p(u_k \cond y_{1:k}) \approx \mathrm{N}\big( u_k \cond m_k, P_k \big)$, $p(u_k \cond y_{1:T}) \approx \mathrm{N}\big( u_k \cond m^s_k, P^s_k \big)$, and $p(u_0) = \mathrm{N}(u_0 \cond m_0, P_0)$. The Gaussian filter computes the filtering estimates $\lbrace m_k, P_k \rbrace_{k=1}^T$ by
	\begin{equation*}
		\begin{split}
			m^-_k &= \int  \Phi(u_{k-1})\, \mathrm{N}(u_{k-1} \cond m_{k-1}, P_{k-1}) \diff u_{k-1}, \\
			P^-_k &= \int  \big(\Omega(u_{k-1}) + \Phi(u_{k-1}) \, \Phi(u_{k-1})^\trans\big)\\
			&\qquad\times\mathrm{N}(u_{k-1} \cond m_{k-1}, P_{k-1}) \diff u_{k-1} - m^-_k \, (m^-_k)^\trans, \\
			S_k &= H \, P^-_k \, H^\trans + \Xi, \\
			K_k &= P^-_k \, H^\trans \, / \,S_k, \\
			m_k &= m^-_k + K_k \, (y_k - H \, m^-_k), \\
			P_k &= P^-_k - K_k \, S_k \, K_k^\trans,
		\end{split}
	\end{equation*}
	for $k=1,2,\ldots, T$. Define $m^s_T\coloneqq m_T$ and $P^s_T\coloneqq P_T$. The Gaussian smoother computes the smoothing estimates $\lbrace m^s_k, P^s_k \rbrace_{k=1}^{T-1}$ by
	\begin{equation*}
		\begin{split}
			D_{k+1} &= \int u_k \, \Phi(u_k)^\trans \mathrm{N}(u_k \cond m_k, P_k) \diff u_k - m_k \, (m^-_{k+1})^\trans,\\
			G_k &= D_{k+1} \, \big( P^-_{k+1} \big)^{-1}, \\
			m^s_k &= m_k + G_k \, (m^s_{k+1} - m^-_{k+1}), \\
			P^s_k &= P_k + G_k \, (P^s_{k+1} - P^-_{k+1}) \, G_k^\trans,
		\end{split}
	\end{equation*}
	for $k=T-1, T-2, \ldots, 1$. Depending on the application, we can use, for example, the first-order Taylor expansion (i.e., EKFSs) or sigma-points (e.g., unscented transform) to approximate the integrals above. The (approximate) negative log-likelihood for optimising the parameter $\theta$ (e.g., $\lambda, b, \ell$, and $\sigma$) is given by
	\begin{equation}
		-\sum^T_{k=1}\log \mathrm{N}\bigl(y_k \cond H \, m^-_k, S_k; \theta\bigr).
		\label{equ:gfs-log-likelihood}
	\end{equation}
\end{algorithm}

The negative log-likelihood in~\eqref{equ:gfs-log-likelihood} depends on the parameters $\theta$ through the filtering recursions which in turn, are also complicated functions of the parameters. Deriving closed-form expressions for the gradients (or Hessian for that matter) as required for numerical optimisation is thus extremely challenging. Fortunately, however, it is also unnecessary because they can be computed efficiently using well-established techniques for automatic differentiation, as made available in a number of popular software libraries, such as JAX and TensorFlow. 

\subsection{Modelling multiple chirps}
\label{sec:harmonic-chirp}
In reality, we also often encounter chirp signals with harmonic frequency components. These signals are of the form
\begin{equation}
	\sum^J_{j=1}\alpha_j(t)\sin\biggl(\phi_{0, j} +  2 \,\pi \, j\int^t_0  f(s) \diff s \biggr) + \xi_k, 
	\label{equ:chirp-harmonic}
\end{equation}
where $J$ is the number of harmonics (including the fundamental IF), and $\alpha_j$ and $\phi_{0, j}$ are the $j$-th instantaneous amplitude and initial phase, respectively. We can handily build a model for the harmonic chirp signals by extending the GP-SDE in~\eqref{equ:sde-chirp}. More specifically, we only need to duplicate the chirp SDE for $X$ by $J$ times independently, which gives $X_1, X_2, \ldots, X_J$, and then stack them together with the SDE of $\overline{V}$. This results in an augmented SDE
\begin{equation*}
	\begin{split}
		\diff \cu{U}(t) &= \cu{A}(\cu{U}(t)) \diff t + \cu{B} \diff \cu{W}(t), \\
		Y_k &= \cu{H} \, \cu{U}(t_k) + \xi_k, 
	\end{split}
\end{equation*}
where
\begin{equation*}
	\begin{split}
		\cu{U}(t) &= 
		\begin{bmatrix}
			X_1(t) & \cdots & X_J(t) & \overline{V}(t)
		\end{bmatrix}^\trans \in \R^{2\,J + 2}, \\ 
	\cu{A}(\cu{U}(t)) &= \mathrm{blkdiag}(A_1(\cu{U}(t)), \ldots, A_J(\cu{U}(t)), M) \, \cu{U}(t),  \\
	A_j(\cu{U}(t)) &= 
	\begin{bmatrix}
		-\lambda_j & -2 \, \pi \, j \, g(V(t)) \\
		2 \, \pi \, j \, g(V(t)) & -\lambda_j
	\end{bmatrix}, \\
	\cu{B} &= \mathrm{blkdiag}(b, b, \ldots, b, b, L) \in\R^{(2 \, J + 2) \times (2 \, J + 2)}, \\
	\cu{W}(t) &= 
	\begin{bmatrix}
		W_X(t) & \cdots & W_X(t) & W_V(t)
	\end{bmatrix}^\trans \in \R^{2\,J + 2}, \\
	\cu{H} &= 
	\begin{bmatrix}
		0 & 1 & \cdots & 0 & 1 & 0 & 0
	\end{bmatrix} \in\R^{2\,J + 2} .
	\end{split}
\end{equation*}
The same routine also works for chirp signals that have multiple fundamental IFs. Suppose that the chirp signal at hand has $R$ fundamental IFs $f_1, f_2, \ldots, f_R$. Then we additionally duplicate the SDE of $\overline{V}$ by $R$ times independently, which gives $\overline{V}_1, \overline{V}_2, \ldots, \overline{V}_R$. The resulting new state $\cu{U}(t) \in \R^{(2 \, J + 2) \, R}$ is a vector formed by stacking $X_{1, 1}(t), \ldots, X_{1, J}(t), \ldots, X_{R, 1}(t),  \ldots ,  X_{R, J}(t), \, \overline{V}_1(t), \ldots, \\\overline{V}_R(t)$. Although the stochastic filters and smoothers still apply, the dimension of the state now grows to $(2 \, J + 2) \, R$ which might be computationally demanding when $J$ and $R$ are large.

\section{Error analysis}
\label{sec:error-analysis}
In this section, we analyse the mean-square error of the IF estimates that follow from using the GP-SDE model introduced in Section~\ref{sec:if-as-gp}. Specifically, we first derive an upper bound of the estimation error to show that the error stays finite in time. Then, we compute a Cram\'{e}r--Rao lower bound for the error which bounds the best estimation result that any estimator can achieve.

\subsection{Mean-square upper bound}
Recall that our IF estimation method amounts to solving a filtering and smoothing problem, and that we have a plenty of filters and smoothers to choose from. Therefore, we now narrow down our scope and focus on the particular, but useful, Gaussian-based filter defined in Algorithm~\ref{alg:gaussian-filter-smoother}. In the literature, there are already a number of mean-square upper bounds of Gaussian filters~\cite{Toni2020}. However, these results are not always informative to our application because they assume that the underlying state-space models are correct. In other words, these classical results are not concerned with whether the state-space model is realistic or not. Thus, to take this into account, we analyse the estimation error when we input the estimator with measurements from a given class of chirp signals and IFs. Formally, the chirp measurement we consider is of the form\looseness=-1
\begin{equation}
	Y_k = \sin\biggl(2 \, \pi \int^{t_k}_0 f(s) \diff s\biggr) + \xi_k,
	\label{equ:err-bound-attack-measurement}
\end{equation}
where $f$ is the true given IF function that we want to estimate. When we input the measurements $\lbrace Y_1, Y_2,\ldots, Y_k \rbrace$ to the estimator, we would like to find a finite bound on the mean squared error
\begin{equation*}
	\expecbig{\abs{f(t_k) - g(H_V \, m_k)}^2},
\end{equation*}
where $H_V$ is the operator that extracts $V$'s estimate from $m_k$ (e.g., $H_V=\begin{bmatrix}0 & 0 & 1 & 0\end{bmatrix}$ when using the Mat\'{e}rn $3 \, / \, 2$ prior). However, this error is difficult to analyse due to the non-linear bijection $g$ wrapping the estimates, and therefore, we transform the analysis into the domain of $V$. The mean squared error that we are interested in now becomes
\begin{equation*}
	\expecbig{\abs{ g^{-1}(f_k) - H_V \, m_k }^2}.
\end{equation*}

In order to carry out the analysis, we also need to fix a representation of $\Phi$ and $\Omega$ in Algorithm~\ref{alg:gaussian-filter-smoother}. Thanks to the hierarchical structure of the SDE~\eqref{equ:sde-chirp}, we can apply the locally conditional discretisation (LCD)~\cite{Zhao2021Thesis}. This LCD scheme approximates $\Phi$ and $\Omega$ by
\begin{align}
	\Phi(u_{k-1})   & =
	\mathrm{blkdiag}\bigg(%
	\expp^{-\Delta_k \, \lambda}\begin{bmatrix}
		\cos(\phi_{k-1}) & -\sin(\phi_{k-1}) \\
		\sin(\phi_{k-1}) & \cos(\phi_{k-1})
	\end{bmatrix}, \nonumber                                                           \\
	                & \qquad\qquad\quad\,\,                                                                                     %
	\expp^{\Delta_k \, M}
	\bigg) \, u_{k-1},\nonumber                                                                                                 \\
	\Omega(u_{k-1}) & = \Omega_{k-1} = \mathrm{blkdiag}\big( \Sigma(\Delta_k), \Lambda(\Delta_k)\big),  \label{equ:disc-lcd} \\
	\phi_{k-1}      & \coloneqq \Delta_k \, 2 \, \pi \, g(H_V \, u_{k-1}),\nonumber                                             \\
	\Delta_k        & \coloneqq t_{k} - t_{k-1}.\nonumber
\end{align}
The exact formulae of $\expp^{\Delta_k \, M}$ and $\Lambda(\Delta_k)$ under the Mat\'{e}rn covariance function in~\eqref{equ:sde-matern32} are shown in Appendix~\ref{append:matern32}. It is worth noting again that the work in~\cite{Scala1996} is a special case of~\eqref{equ:disc-lcd} under this particular LCD discretisation scheme and a specific parameter choice. 

The main result (i.e., the error upper bound) is given in Theorem~\ref{thm:error}. In order to arrive at the result, we use the following assumptions.

\begin{assumption}[Regularity of IF and prior]
	\label{assump:regularity-prior-if}
	There exist a non-negative function $z$ and constant $c \geq 0$ such that at each $k$,
	\begin{equation*}
		\absbig{g^{-1}(f_k) - \overline{H}_V \, \expp^{\Delta_k \, M} \, x}^2 \leq z(\Delta_k) \absbig{g^{-1}(f_{k-1}) - \overline{H}_V \, x}^2 + c
	\end{equation*}
	for all $x\in\R^{d_v}$, where $\overline{H}_V \in \R^{d_v}$ is defined in Section~\ref{sec:model-formulation}.
\end{assumption}

Assumption~\ref{assump:regularity-prior-if} confines the class of IF functions and bijections that this error analysis is dealing with. Essentially, this assumption means that for any pair of $(x, g^{-1}(f_{k-1}))$ at $t_{k-1}$, the error (between $x$ and the true IF value $g^{-1}(f_{k-1})$) should not grow significantly (controlled by $z$ and $c_k$) to time $t_k$ when the prior makes a prediction based on $x$. This makes sense because one must choose the prior for the IF appropriately in order to best describe the latent IF. Note that we also have the freedom to choose a positive bijection $g$ to satisfy this assumption.\looseness=-1

To see that Assumption~\ref{assump:regularity-prior-if} is not restrictive, we can enumerate a few realistic examples that satisfy the assumption. Suppose that the true IF is $f(t) = \sin(t) + \epsilon$ for some base frequency $\epsilon>1$, and that we choose $g(\cdot) = \exp(\cdot)$ and $M=-1$ (i.e., a $d_v=1$ Mat\'{e}rn $1 \, / \, 2$ prior). Then there exists a constant $c$ (independent of both $k$ and $\Delta_k$) and a function $z(\Delta_k) = \expp^{\Delta_k \, M}$ that satisfy this assumption. Moreover, it is not hard to manipulate $\Delta_k$ or $M$ so that $z(\Delta_k) < 1/3$ in order to obtain a contractive error bound as in Corollary~\ref{corollary:simple-bound}.

\begin{assumption}
	\label{assump:trace-class}
	There exist constants $c_P$ and $c_{\overline{P}}$ such that $\trace{P_k} \leq c_P$ and $\norm{P^-_k}_2^2 \leq c_{\overline{P}}$ almost surely for all $k\geq 1$.
\end{assumption}

Assumption~\ref{assump:trace-class} requires that the filtering and prediction covariances are finitely bounded. This assumption is indeed strong, in the sense that it is hard to verify in practice. On the other hand, relaxing this assumption is difficult because the evolution of these covariances is non-linearly coupled with that of the posterior mean estimates and measurements, while the evolution of the posterior mean estimates depends on the covariances too. This is a common problem in analysing the errors of non-linear filters, and this type of assumptions is routinely used in the literature~\cite{Toni2020}.

\begin{lemma}
	\label{lemma:expectation-of-phi}
	For any mean $\mu\in\R^{d_u}$ and covariance $\Theta\in\R^{{d_u}\times {d_u}}$, the conditional mean $\Phi$ is such that
	\begin{equation}
		\int \norm{\Phi(u)}_2^2 \, \mathrm{N}(u \cond \mu, \Theta) \diff u  \leq \norm{\mu}_2^2 + \trace{\Theta}.
	\end{equation}
\end{lemma}
\begin{proof}
	This follows from the fact that $\norm{\Phi(u)}_2 \leq \norm{u}_2$ for all $u\in\R^{d_u}$ in~\eqref{equ:disc-lcd}.
\end{proof}

\begin{theorem}
	\label{thm:error}
	Suppose that Assumptions~\ref{assump:regularity-prior-if} and~\ref{assump:trace-class} hold. Then for every $k\geq 1$, the error is such that
	\begin{equation}
		\begin{split}
			&\expecbig{\abs{ g^{-1}(f_k) - H_V \, m_k }^2} \\
			&\leq e_0 \prod_{j=1}^k 3 \, z(\Delta_j)
			+ \gamma \sum^k_{j=1}\prod^{j-1}_{i=1} 3 \, z(\Delta_{k-i+1}) \\
			&\quad+ \sum^k_{j=1}\zeta_{k-j+1}\prod^{j-1}_{i=1} 3 \, z(\Delta_{k-i+1}),
			\label{equ:err-bound}
		\end{split}
	\end{equation}
	where $e_0 \coloneqq \expecbig{\abs{ g^{-1}(f_0) - H_V \, m_0 }^2}$ stands for the initial error, and
	\begin{equation}
		\begin{split}
			\gamma &\coloneqq 3 \, c + \frac{6 \, c_{\overline{P}}}{(c_\Sigma + \Xi)^2},\\
			\zeta_k &\coloneqq (2 \, c_K)^k \, \frac{3 \, c_{\overline{P}} \, \big(\norm{m_0}_2^2 + \trace{P_0}\big)}{(c_\Sigma + \Xi)^2} \\
			&\quad+\frac{3 \, c_{\overline{P}}}{(c_\Sigma + \Xi)^2} \, \bigg(2 \, \frac{c_{\overline{P}} \, (1 + \Xi)}{(c_\Sigma + \Xi)^2} + c_P\bigg) \sum^{k-1}_{j=0} (2 \, c_K)^j, \\
			c_\Sigma &\coloneqq \inf_{j\geq 1} \Sigma(\Delta_j).
		\end{split}
	\end{equation}
	Note that we define $\prod^0_{i=1} \coloneqq 1$ and $\sum^0_{j=0}\coloneqq 0$.
\end{theorem}
\begin{proof}
	Thanks to Assumption~\ref{assump:trace-class}, the Kalman gain is such that
	\begin{equation*}
		\begin{split}
			\norm{K_k}_2^2 &= \norm{P^-_k \, H^\trans \, / \, S_k}_2^2 \leq  \frac{c_{\overline{P}}}{(c_\Sigma + \Xi)^2},
		\end{split}
	\end{equation*}
	and then that $\norm{I - K_k \, H}_2^2 \leq c_K$ for some constant $c_K$ (determined by $c_{\overline{P}}$, $c_\Sigma$, and $\Xi$), almost surely.

	By applying the triangle inequality on the mean squared error, we arrive at a bound composed of three residuals:
	\begin{align}
		 & \expecbig{\absbig{g^{-1}(f_k) - H_V \, m_k}^2} \nonumber                                                                              \\
		 & = \expecbig{\absbig{g^{-1}(f_k) - H_V \, m^-_k + H_V \, K_k \, (Y_k - H \, m^-_k)}^2} \label{equ:err-bound-sec-residual}              \\
		 & \leq 3 \expecbig{\absbig{g^{-1}(f_k) - H_V \, m^-_k}^2} \nonumber                                                                     \\
		 & \quad+ 3 \expecbigg{\absbigg{H_V \, K_k \,\bigg(\sin\bigg(\int^{t_k}_0 2 \, \pi \, f(s)\diff s\bigg) - H \, m^-_k\bigg)}^2} \nonumber \\
		 & \quad+ 3 \expecbig{\abs{H_V \, K_k \,\xi_k}^2}. \nonumber
	\end{align}
	The last residual satisfies $\expecbig{\abs{H_V \, K_k \,\xi_k}^2} \leq c_{\overline{P}} \, / \, (c_\Sigma + \Xi)^2$. We can establish a bound to the first residual by applying Assumption~\ref{assump:regularity-prior-if}, resulting in
	\begin{equation*}
		\begin{split}
			&\expecbig{\absbig{g^{-1}(f_k) - H_V \, m^-_k}^2} \\
			&\leq z(\Delta_k)  \expecbig{\absbig{g^{-1}(f_{k-1}) - H_V \, m_{k-1}}^2} +c.
		\end{split}
	\end{equation*}
	As for the second residual, we have
	\begin{equation}
		\begin{split}
			&\expecbigg{\absbigg{H_V \, K_k \,\bigg(\sin\bigg(\int^{t_k}_0 2 \, \pi \, f(s)\diff s\bigg) - H \, m^-_k\bigg)}^2} \\
			&\leq \frac{c_{\overline{P}}}{(c_\Sigma + \Xi)^2}\expecbigg{\absbigg{\sin\bigg(\int^{t_k}_0 2 \, \pi \, f(s)\diff s\bigg) - H \, m^-_k}^2} \\
			&\leq  \frac{c_{\overline{P}} \, \big(1 + \expecbig{\norm{m^-_k}_2^2}\big)}{(c_\Sigma + \Xi)^2}.
			\label{equ:theorem-second-res}
		\end{split}
	\end{equation}
	Lemma~\ref{lemma:expectation-of-phi} and Assumption~\ref{assump:trace-class} imply that
	\begin{equation}
		\begin{split}
			\expecbig{\norm{m^-_k}_2^2} &\leq \expecbig{\norm{m_{k-1}}_2^2 + \trace{P_{k-1}}} \\
			&\leq 2 \expecbig{\norm{(I - K_{k-1} \, H) \, m^-_{k-1}}_2^2} \\
			&\quad+ 2 \expecbig{\norm{K_{k-1} \, Y_{k-1}}_2^2} + c_P \\
			&\leq 2 \, c_K \expecbig{\norm{m^-_{k-1}}_2^2} + 2 \, \frac{c_{\overline{P}} \, (1 + \Xi) }{(c_\Sigma + \Xi)^2} + c_P.
		\end{split}
	\end{equation}
	By unrolling the recursion of $\expecbig{\norm{m^-_k}_2^2}$ from the equation above for $k\geq 1$, we obtain
	\begin{equation}
		\begin{split}
			\expecbig{\norm{m^-_k}_2^2} &\leq (2 \, c_K)^k \, \big(\norm{m_0}_2^2 + \trace{P_0}\big) \\
			&\quad+ \bigg(2 \, \frac{c_{\overline{P}} \, (1 + \Xi) }{(c_\Sigma + \Xi)^2} + c_P\bigg) \sum^{k-1}_{j=0} (2 \, c_K)^j,
			\label{equ:theorem-m-recursion}
		\end{split}
	\end{equation}
	noting that the initial $\expecbig{\norm{m^-_1}_2^2} =  \norm{m^-_1}_2^2 \leq \norm{m_0}_2^2 + \trace{P_0}$. We can then substitute~\eqref{equ:theorem-m-recursion} back into~\eqref{equ:theorem-second-res}. Finally, by putting it all together, we have
	\begin{equation}
		\begin{split}
			&\expecbig{\abs{g^{-1}(f_k) - H_V \, m_k}^2} \\
			&\leq 3\,z(\Delta_k)  \expecbig{\abs{g^{-1}(f_{k-1}) - H_V \, m_{k-1}}^2} + 3\, c \\
			&\quad+ \frac{3 \, c_{\overline{P}}}{(c_\Sigma + \Xi)^2} + (2 \, c_K)^k \, \frac{3 \, c_{\overline{P}} \, \big(\norm{m_0}_2^2 + \trace{P_0}\big)}{(c_\Sigma + \Xi)^2} \\
			&\quad+\frac{3 \, c_{\overline{P}}}{(c_\Sigma + \Xi)^2} \, \Big(2 \, \frac{c_{\overline{P}} \, (1 + \Xi)}{(c_\Sigma + \Xi)^2} + c_P\Big) \sum^{k-1}_{j=0} (2 \, c_K)^j \\
			&\quad+ 3 \, \frac{c_{\overline{P}}}{(c_\Sigma + \Xi)^2}
		\end{split}
		\label{equ:thm-unroll-recursion}
	\end{equation}
	which is a recursion of the error. Unrolling the recursion for $k\geq 1$ concludes the result.
\end{proof}

Theorem~\ref{thm:error} provides an upper bound of the estimation error in the mean square sense. This bound is dominated by the function $z$ and the constant $c_K$, in the way that the bound is contractive as long as $z$ and $c_K$ are not too large. This result makes sense because it reflects that the IF prior should be chosen well enough to model the true IF (see Assumption~\ref{assump:regularity-prior-if}), and that the Kalman gain should be finite too.

In the following corollary, we show that the error bound from Theorem~\ref{thm:error} can be simplified to a contractive one when $z$ and $c_K$ meet a certain criterion.

\begin{corollary}
	\label{corollary:simple-bound}
	Following Theorem~\ref{thm:error}, suppose that $c_K < 1 \, / \, 2$, and that $z(\Delta_j) \leq c_z < 1 \, / \, 3$ for all $j=1,2,\ldots,k$, then the error bound in~\eqref{equ:err-bound} reduces to
	\begin{equation}
		\expecbig{\abs{ g^{-1}(f_k) - H_V \, m_k }^2} \leq (3 \, c_z)^k \, e_0 + \frac{\gamma + \overline{\zeta}}{1 - 3 \, c_z} ,
	\end{equation}
	where
	\begin{equation*}
		\overline{\zeta} \coloneqq \frac{3 \, c_{\overline{P}}}{(c_\Sigma + \Xi)^2}\bigg( \norm{m_0}_2^2 + \trace{P_0} +  \frac{2 \, \frac{c_{\overline{P}} \, (1 + \Xi)}{(c_\Sigma + \Xi)^2} + c_P}{1 - 2 \, c_K} \bigg).
	\end{equation*}
\end{corollary}
\begin{proof}
	Recall the identity that $\sum^k_{j} a^j < \frac{1}{1 - a}$ for every number $0<a<1$ and indexes $k=0,1,\ldots$ By applying the assumptions and this identity to~\eqref{equ:err-bound} we conclude the result.
\end{proof}

It is worth noting that the error bound in the main theorem does not reflect how well the chirp prior copes with the chirp signal. Indeed, the bound shows that the prior for the IF should be well-chosen through $z$, but it does not explicitly reveal how well the harmonic SDE~\eqref{equ:harmonic-sde} models the chirp signal in~\eqref{equ:err-bound-attack-measurement}. This is due to~\eqref{equ:err-bound-sec-residual} where we used a worst-case triangle bound on the residual between the chirp prediction (i.e., $H \, m^-_k$) and the true chirp signal (i.e., $\sin(\int^{t_k}_0 2 \, \pi \, f(s) \diff s)$): they are both bounded. Eventually, this residual error is not directly reflected in the final error bound but it is instead implicitly contained in that of the constant $c_K$. We believe that the error bound can be improved given a tighter bound on the chirp residual. This is a worthwhile future work. 

\subsection{Cram\'{e}r--Rao lower bound}
\label{sec:crlb}
In this section, we compute a Cram\'{e}r--Rao lower bound (CRLB) for the mean-square error. However, recall that the true state $U(t_k)$ is not deterministic but a random variable, and thus the CRLB in the classical definition does not apply. The stochastic counterpart of the classical CRLB is given in~\cite{Tichavsky1998}, and is known as the posterior CRLB. If $m_k$ is the filtering estimate that depends on the measurements $Y_{1:k}$, then the posterior CRLB is 
\begin{equation*}
	\expecbig{(m_k - U(t_k)) \, (m_k - U(t_k))^\trans} \geq \mathcal{I}_k^{-1}, 
\end{equation*}
where $\mathcal{I}_k^{-1}$ is the $k$-th sub-matrix of the inverse of the Fisher information matrix $\mathcal{I}_{1:k}$ defined by 
\begin{equation}
	\mathcal{I}_{1:k} \coloneqq -\expecbig{\nabla\nabla\vphantom{\nabla}_{u_{1:k}}^\trans \!\log p(y_{1:k}, u_{1:k})},
	\label{equ:crlb-I}
\end{equation}
where $p(y_{1:k}, u_{1:k}) \coloneqq p_{Y_{1:T}, U_{1:T}}(y_{1:k}, u_{1:k})$ is the joint probability density function of $Y_{1:T}$ and $U_{1:T}$, and $\nabla\nabla_{u_{1:k}}^\trans$ denotes the Hessian matrix with respect to the variable $u_{1:k}$. Moreover, thanks to the recursion relation in~\cite{Tichavsky1998}, we do not need to explicitly compute the full-rank Hessian matrix $\mathcal{I}_{1:k}$; we can recursively compute $\mathcal{I}_k$ for $k=1,2,\ldots$ from the initial $\mathcal{I}_0 \coloneqq -\expec{\nabla\nabla\vphantom{\nabla}_{u_0}^\trans \log p(u_0)}$ with a low-cost computation. The CRLB is not analytically tractable due to the expectations in~\eqref{equ:crlb-I}, and hence, we numerically compute these expectations by 1,000,000 independent Monte Carlo simulations of the model in~\eqref{equ:sde-chirp}.

\begin{figure}
	\centering
	\includegraphics[width=.99\linewidth]{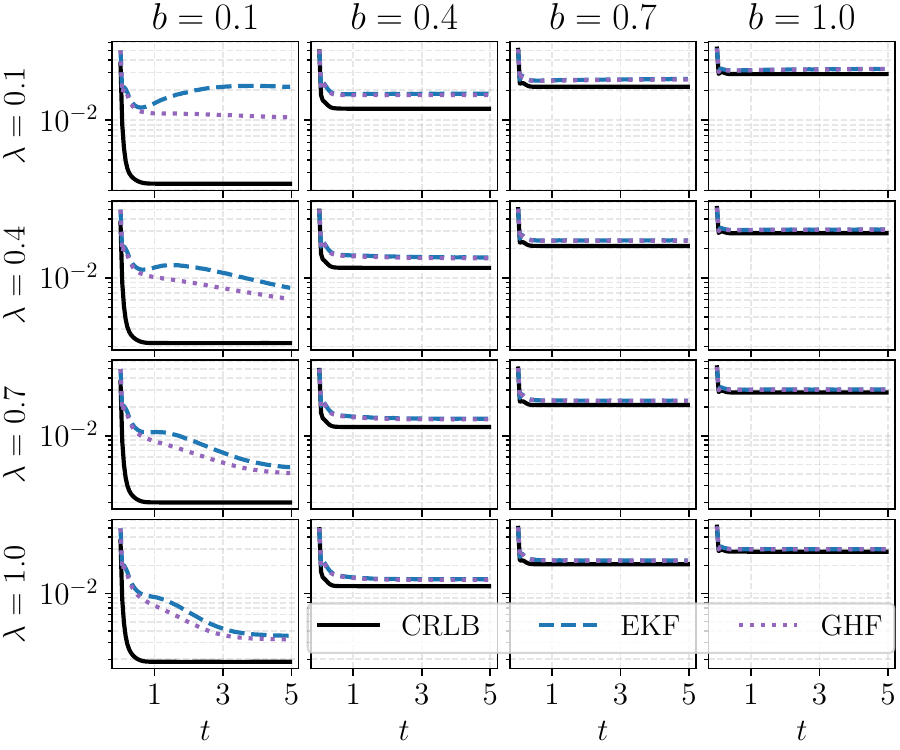}
	\includegraphics[width=.99\linewidth]{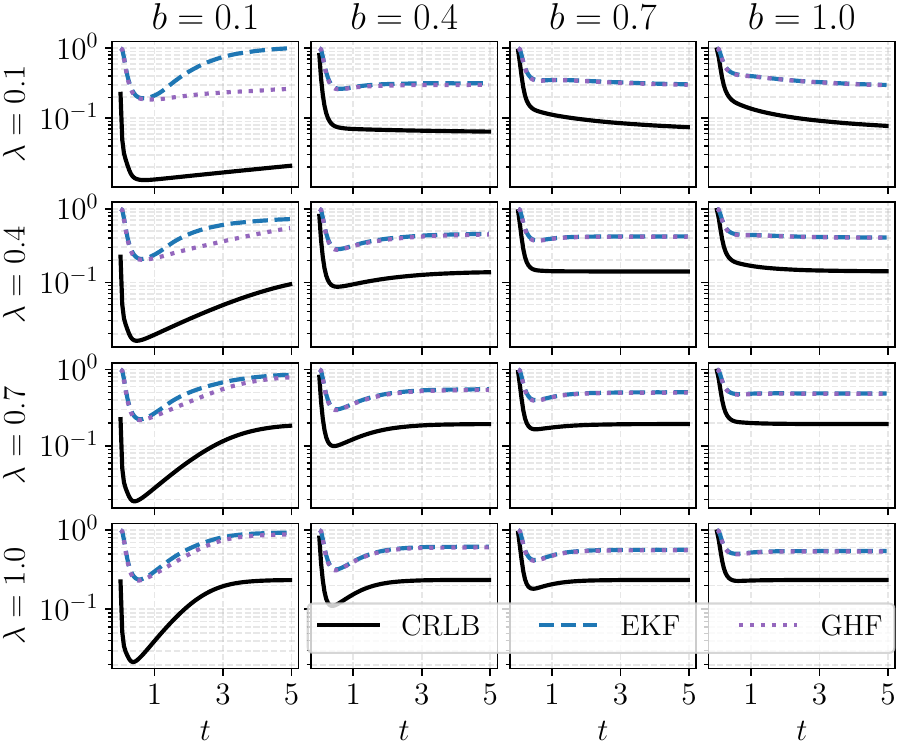}
	\caption{The CRLBs of the GP-SDE model in~\eqref{equ:sde-chirp} with different parameter settings. The top and bottom blocks show the CRLBs of the chirp (i.e., $X(t)$) and IF (i.e., $V(t)$) components in the state, respectively.}
	\label{fig:crlbs}
\end{figure}

The CRLB of our model in~\eqref{equ:sde-chirp} is plotted in Figure~\ref{fig:crlbs}. The first important result we see from the figure is that the CRLB converges as $t\to \infty$ which is expected~\cite[see,][pp. 1389]{Tichavsky1998}. In addition, the figure shows CRLBs with different values of $\lambda$ and $b$, since these two parameters are key in characterising the chirp state, while other parameters are fixed by $\ell=1$, $\sigma=1$, $\delta = 0.1$, and $\Xi=0.1$. We detail the results in the following.

We see from the top block of Figure~\ref{fig:crlbs} that the CRLB for the chirp state decreases as the parameter $b$ decreases. This is expected, because $b$ controls the stochastic volatility of the chirp state, and hence, the signal-to-noise-ratio. But on the other hand, this does not mean that $b$ should always be very small or even zero in applications, because real-world chirp signals have random amplitudes which demand a non-zero $b$. The parameter $\lambda$, however, does not significantly affect the CRLB of the chirp state.

The bottom block of Figure~\ref{fig:crlbs} shows the CRLB for the IF state. We see that as $\lambda$ increases, in particular when $b$ is small, the CRLB decreases. This is intuitive because $\lambda$ controls the damping factor of the chirp state. The mean of the chirp state converges to zero at the speed of $\exp(-\lambda \, t)$, and hence, when $t$ is large, the mean almost amounts to zero, which makes it hard to identify the underlying IF. As for the parameter $b$, it does not substantially influence the CRLB of the IF state. 

We also gauge two reference methods relative to the CRLBs. These methods are the extended Kalman filter (EKF) and the Gauss--Hermite filter (GHF) which are commonly used instances of Algorithm~\ref{alg:gaussian-filter-smoother}. In the top block of Figure~\ref{fig:crlbs} we see that the two estimators are close to the CRLB for estimating the chirp state as the time increases, while GHF is slightly better than EKF. On the other hand, in the bottom block of the same figure, we see that the IF estimates have a noticeable distance to the CRLB regardless of the parameter values. This is due to two reasons: the EKF and GHF methods are approximate estimators (even biased), and the bound in~\eqref{equ:crlb-I} is not tight. There are a number of posterior CRLBs tighter~\cite[see, e.g.,][]{Fritsche2014} than the used $\mathcal{I}_k^{-1}$, but unfortunately, these bounds are either computationally expensive or require exact knowledge of the filtering distribution.

\begin{table*}[]
	\centering
	\caption{Means, medians, and minimums (from 100 MC runs) of the IF estimation RMSEs for $J=1$. The methods below the dashed line use the proposed GP-SDE model. Bold numbers represent the best in each of their columns.}
	\label{tbl:rmse-reuslts}
	\begin{tabular}{@{}rrrrrrrrrr@{}}
		\toprule
		\multirow{2}{*}{Method / RMSE ($\times 10^{-1}$)} & \multicolumn{3}{r}{$\alpha(t)=1$} & \multicolumn{3}{r}{$\alpha(t)=\expp^{-0.3 \,t}$} & \multicolumn{3}{r}{$t\mapsto\alpha(t)$ is a random process} \\ \cmidrule(l){2-4} \cmidrule(l){5-7} \cmidrule(l){8-10} 
		& mean $\pm$ std.             & median  & min.  & mean                   & median      & min.       & mean                    & median       & min.        \\ \midrule
		Hilbert transform       & $7.13 \pm 2.35$  & $6.37$    & $5.64$ & $11.74 \pm 11.06$      & $9.02$        & $7.06$      & $54.63 \pm 25.58$       & $51.22$        & $6.04$       \\
		Spectrogram             & $1.53 \pm 0.08$  & $1.53$    & $1.30$ & $1.82 \pm 0.18$        & $1.83$        & $1.41$      & $8.17 \pm 4.31$         & $6.84$         & $1.78$       \\
		Polynomial MLE          & $8.87 \pm 0.09$  & $8.87$    & $8.63$ & $8.90 \pm 0.13$        & $8.88$        & $8.58$      & $10.01 \pm 4.33$        & $9.16$         & $4.78$       \\
		ANF~\cite{Niedzwiecki2011}                     & $2.13 \pm 0.16$  & $2.11$    & $1.80$ & $3.05 \pm 0.31$        & $3.04$        & $2.44$      & $37.77 \pm 23.57$       & $33.52$        & $1.83$       \\
		EKFS MLE on~\eqref{equ:old-ekf-ss}~\cite{Scala1996}         & $1.09 \pm 0.20$  & $1.08$    & $0.69$ & $19.53 \pm 18.14$      & $3.40$        & $2.12$      & $39.85 \pm 17.61$       & $40.19$        & $1.25$       \\
		GHFS MLE on~\eqref{equ:old-ekf-ss}~\cite{Scala1996}         & $0.67 \pm 0.17$  & $0.62$    & $0.39$ & $3.48 \pm 7.15$        & $1.92$        & $1.01$      & $38.15 \pm 19.51$       & $40.06$        & $1.48$       \\
		FastNLS~\cite{Nielsen2017}                 & $1.08 \pm \mathbf{0.05}$  & $1.08$    & $0.94$ & $1.27 \pm \mathbf{0.11}$        & $1.26$        & $1.04$      & $10.90 \pm 6.15$        & $10.39$        & $1.04$       \\
		FHC~\cite{Jensen2017}                     & $0.91 \pm 0.06$  & $0.90$    & $0.77$ & $1.16 \pm 0.11$        & $1.15$        & $0.94$      & $12.96 \pm 6.70$        & $12.16$        & $1.00$       \\
		KPT MLE~\cite{Shi2017}                 & $1.49 \pm 0.17$  & $1.47$    & $1.10$ & $1.83 \pm 0.24$        & $1.79$        & $1.40$      & $21.63 \pm 21.40$       & $12.35$        & $1.17$       \\\cdashlinelr{1-10}
		EKFS MLE                & $0.70 \pm 0.17$  & $0.69$    & $\mathbf{0.37}$ & $0.98 \pm 0.24$        & $0.97$        & $\textbf{0.51}$      & $6.38 \pm 7.04$         & $4.16$         & $0.55$       \\
		GHFS MLE                & $\mathbf{0.65} \pm 0.16$  & $\mathbf{0.61}$    & $0.38$ & $\mathbf{0.93} \pm 0.24$        & $\mathbf{0.92}$        & $0.53$      & $5.11 \pm 5.12$         & $\mathbf{3.61}$         & $0.65$       \\
		CD-EKFS MLE             & $1.53 \pm 0.67$  & $1.55$    & $0.37$ & $2.86 \pm 2.26$        & $1.37$        & $0.54$      & $6.34 \pm 7.18$         & $4.09$         & $\mathbf{0.53}$       \\
		CD-GHFS MLE             & $0.72 \pm 0.18$  & $0.68$    & $0.38$ & $1.16 \pm 0.34$        & $1.14$        & $0.56$      & $\mathbf{4.66} \pm \mathbf{3.48}$         & $3.73$         & $0.65$       \\ \bottomrule
	\end{tabular}
\end{table*}

\begin{table}[t!]
	\centering
	\caption{Estimated model parameters from the GHFS MLE method shown in Figure~\ref{fig:estimation}. The parameter $\lambda$ in the second column is given in the bold face since its estimate is very close to the true damping factor $0.3$.}
	\label{tbl:learnt-params}
	\begin{tabular}{@{}llll@{}}
		\toprule Parameters
		 & $\alpha(t)=1$        & $\alpha(t)=\expp^{-0.3 \,t}$  & \begin{tabular}[c]{@{}l@{}}$t\mapsto\alpha(t)$ is a \\  random process\end{tabular} \\ \midrule
		$\lambda$ & $2.06\times 10^{-2}$ & $\mathbf{3.00\times 10^{-1}}$ & $1.05$                     \\
		$b$       & $8.07\times 10^{-5}$ & $5.77\times 10^{-3}$          & $1.05$                     \\
		$\delta$  & $4.51\times10^{-1}$  & $4.56\times10^{-1}$           & $0.99\times10^{-1}$        \\
		$\ell$    & $1.20$               & $1.14$                        & $1.33$                     \\
		$\sigma$  & $4.88$               & $4.86$                        & $5.98$                     \\
		$m_0^V$   & $10.37$              & $10.26$                       & $12.97$                    \\ \bottomrule
	\end{tabular}
\end{table}

\begin{figure*}[t!]
	\centering
	\includegraphics[width=.99\linewidth]{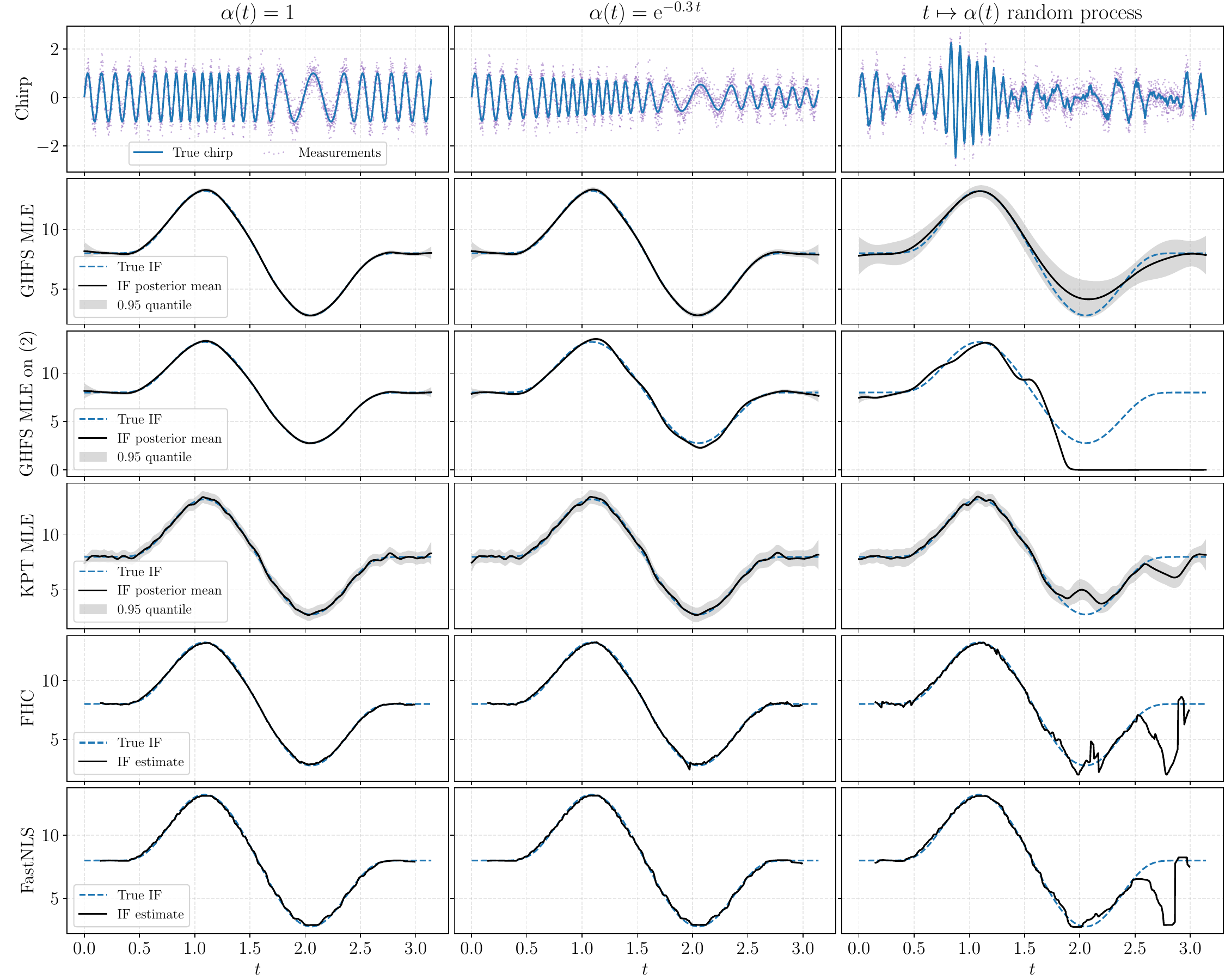}
	\caption{IF estimates ($J=1$) from one MC run. For simplicity, we only plot the estimates from the methods that stand out in Table~\ref{tbl:rmse-reuslts}. }
	\label{fig:estimation}
\end{figure*}

\begin{table*}[]
	\centering
	\caption{Means, medians, and minimums (from 100 MC runs) of the harmonic IF estimation RMSEs for $J=3$. The methods below the dashed line use the proposed GP-SDE model. Bold numbers represent the best in each of their columns.}
	\label{tbl:results-rmse-harmonic}
	\begin{tabular}{@{}rrrrrrrrrr@{}}
		\toprule
		\multirow{2}{*}{Method / RMSE ($\times 10^{-1}$)} & \multicolumn{3}{r}{$\alpha(t)=1$} & \multicolumn{3}{r}{$\alpha(t)=\expp^{-0.3 \,t}$} & \multicolumn{3}{r}{$t\mapsto\alpha(t)$ is a random process} \\ \cmidrule(l){2-4} \cmidrule(l){5-7} \cmidrule(l){8-10} 
		& mean $\pm$ std.             & median  & min.  & mean                   & median      & min.       & mean                    & median       & min.        \\ \midrule
		FastNLS~\cite{Nielsen2017}                 & $2.35 \pm \mathbf{0.05}$  & $2.35$    & $2.23$ & $2.40 \pm \mathbf{0.10}$       & $2.39$         & $2.20$      & $\mathbf{8.23} \pm 4.44$         & $7.27$         & $2.29$       \\
		FHC~\cite{Jensen2017}                     & $0.81 \pm 0.16$  & $0.76$    & $0.71$ & $1.59 \pm 0.42$       & $1.58$         & $0.74$      & $12.58 \pm 5.49$        & $12.76$        & $0.77$       \\
		KPT MLE~\cite{Shi2017}                 & $1.63 \pm 0.44$  & $1.54$    & $0.94$ & $1.63 \pm 1.77$       & $1.45$         & $1.05$      & $42.21 \pm 22.62$       & $37.32$        & $1.23$       \\\cdashlinelr{1-10}
		EKFS MLE                & $\mathbf{0.40} \pm 0.09$  & $\mathbf{0.39}$    & $\mathbf{0.24}$ & $\mathbf{1.00} \pm 1.93$       & $0.57$         & $0.53$      & $19.83 \pm 21.53$       & $10.31$        & $\mathbf{0.43}$       \\
		CKFS MLE                & $0.45 \pm 0.14$  & $0.42$    & $0.24$ & $1.52 \pm 3.07$       & $\mathbf{0.56}$         & $\mathbf{0.28}$      & $9.95 \pm 5.12$         & $\mathbf{3.40}$         & $0.47$          \\ \bottomrule
	\end{tabular}
\end{table*}

\begin{figure}[t!]
	\centering
	\includegraphics[width=.99\linewidth]{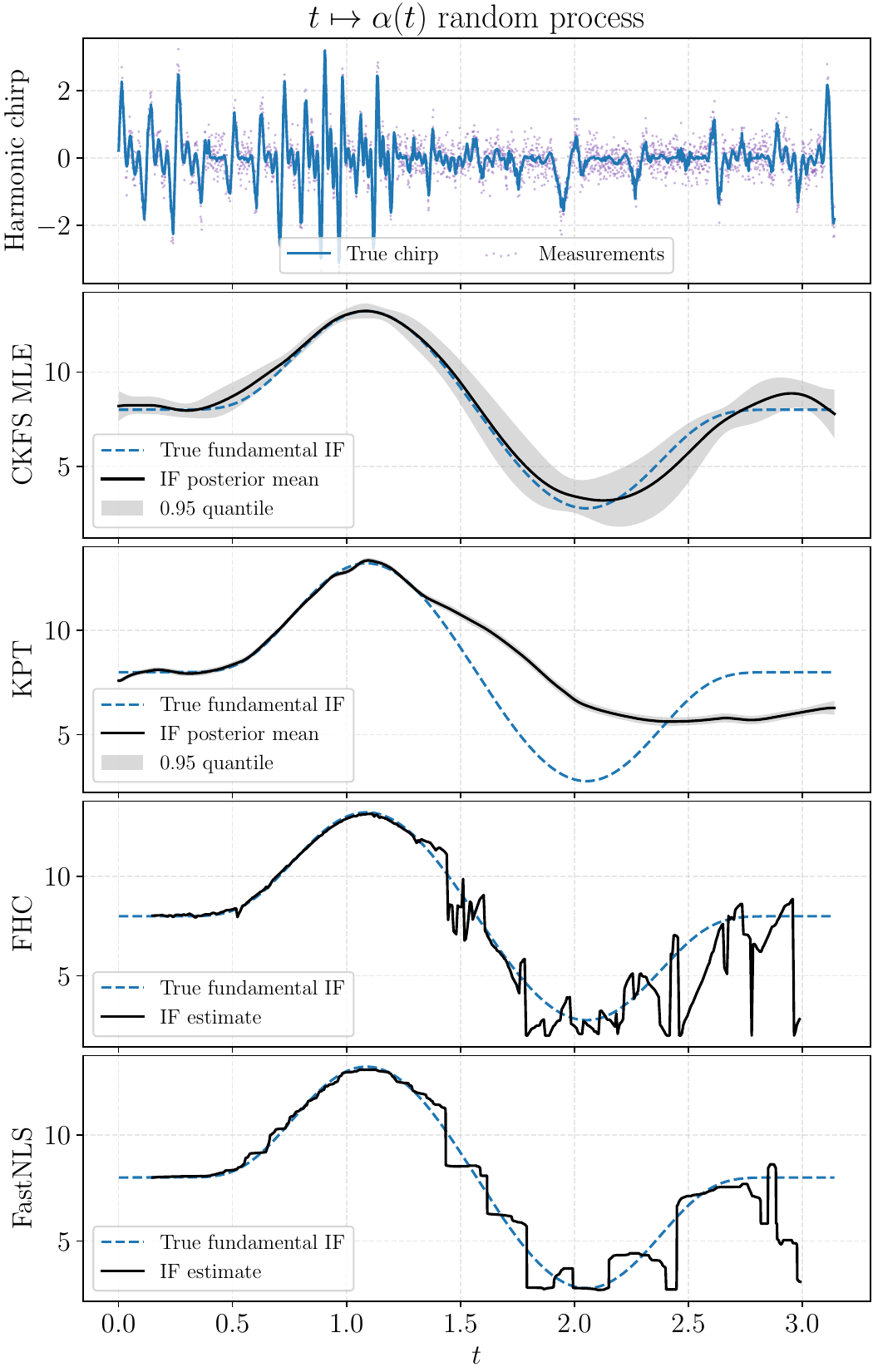}
	\caption{Fundamental IF estimates ($J=3$). These results are taken from one MC run in Table~\ref{tbl:results-rmse-harmonic}. }
	\label{fig:estimation-harmonic}
\end{figure}

\section{Experiments}
\label{sec:experiments}
In this section, we test the performance of the proposed chirp estimation scheme on a synthetic data and compare to several state-of-the-art methods. Furthermore, we apply the method on a gravitational wave data and two European bat calls, to show that our estimation method works on real-world data out-of-the-box. For the sake of reproducibility, our implementations are online available at: \url{https://github.com/spdes/chirpgp}.

\subsection{Synthetic experiment}
\label{sec:experiment-toy}
The synthetic model that we use to test the performance of the proposed scheme is given in the following equation:
\begin{equation}
	\begin{split}
		f(t) &= a \, b \cot(t) \csc(t) \, \expp^{-b \csc(t)} + c, \quad t \in (0, \pi), \\
		Y_k &= \sum^J_{j=1} \alpha(t_k) \sin\biggl(2 \, \pi \, j \int^{t_k}_0 f(s) \diff s\biggr) + \xi_k,
		\label{equ:toymodel}
	\end{split}
\end{equation}
where $f$ is the true IF parametrised by $a=500$, $b=5$, and $c=8$ (which control the function's amplitude, horizontal scale, and offset, respectively), and $J$ is the number of harmonic frequencies. The integral of $f$ (i.e., the phase) is analytically available and it is $a \exp(-b \, / \sin(t)) + c \, t$. Measurements of this chirp signal at any time $t_k$ is represented by a random variable $Y_k$ along with a Gaussian noise $\xi_k \sim \mathrm{N}(0, 0.1)$.

We consider three instantaneous amplitude functions, namely, constant $\alpha(t)=1$, damped $\alpha(t)=\exp(-0.3 \, t)$, and a random $t \mapsto \alpha(t)$. The random $t \mapsto \alpha(t)$ is a path generated from an Ornstein--Uhlenbeck SDE $\diff \alpha(t) = - \alpha(t) \diff t + \diff W(t)$. The aim of using these different amplitude functions is to test if the methods can identify the IF under the nuisance from the chirp amplitude, in particular, the random $t \mapsto \alpha(t)$.

We generate the chirp signal and its measurements using a sampling frequency of $1,000$~Hz (i.e., 3,141 data points in total). The estimation performance is quantified by the root mean squared error (RMSE). Furthermore, we run all the IF estimation methods using 100 independent Monte Carlo (MC) simulations in order to get the statistics of their RMSE results. 

We apply the extended Kalman filter and smoother (EKFS), and the 3rd-order Gauss--Hermite Gaussian filter and smoother (GHFS) -- which are the most popular instances of Algorithm~\ref{alg:gaussian-filter-smoother} -- to solve the proposed model in~\eqref{equ:sde-chirp}. The discretisation of this model is performed using the LCD scheme in~\eqref{equ:disc-lcd}, and the bijection is $g(\cdot)=\log(\exp(\cdot)+1)$. In addition, since our model can be solved in continuous-time, we also apply the continuous-discrete EKFS and GHFS (CD-EKFS and CD-GHFS, respectively)~\cite{SimoGFS2013} to solve it directly without discretising the SDE in~\eqref{equ:sde-chirp}.

The model proposed in~\eqref{equ:sde-chirp} consists of six parameters, that are, $\lambda$, $b$, $\delta$, $\ell$, $\sigma$, and the initial mean of $V$ denoted by $m^V_0$. They are determined via the maximum likelihood estimation (MLE) by minimising the objective function in~\eqref{equ:gfs-log-likelihood} using the L-BFGS optimiser.

We consider the following baseline methods to compare to ours: Hilbert transform; first-order power spectrum (spectrogram), adaptive notch filter (ANF)~\cite{Niedzwiecki2011}, 10-th order polynomial MLE, fast non-linear least square (FastNLS)~\cite{Nielsen2017}, fast MLE for harmonic chirp parameters (FHC)~\cite{Jensen2017}, and the state-space models by~\cite{Scala1996} and~\cite{Shi2017} (called EKFS/GHFS MLE on~\eqref{equ:old-ekf-ss} and KPT MLE, respectively). In particular, the state-space models are arguably the most interesting to compare with, as they are close to our model in~\eqref{equ:sde-chirp}. The parameters and settings of these baseline methods are detailed in Appendix~\ref{append:baseline}.

We first consider the single harmonic $J=1$ case, and we summarise the RMSE results in Table~\ref{tbl:rmse-reuslts}. From this table, it is clear that the proposed IF estimation scheme outperforms all the other methods (in terms of mean, median, and minimum). In particular, when the amplitude function $\alpha$ is a random process, the proposed methods are substantially better than the others. This verifies that the proposed model is a good candidate for modelling chirp signals with random amplitude nuisances. When the amplitude is a constant, however, the margin of using the proposed model over that of~\cite{Scala1996} and FHC is not significant. It is also worth noting that the FastNLS and FHC methods have low standard deviations compared to others, but this changes when the amplitude $\alpha$ is random.

We select the five best methods in Table~\ref{tbl:rmse-reuslts} and plot their IF estimates (from one MC run) in Figure~\ref{fig:estimation}. This figure qualitatively shows that the GHFS MLE method provides the best IF estimates compared to the other four methods. Furthermore, from the figure in the second row and third column, we can see that the 0.95 quantile produced by GHFS MLE is reasonable: the true IF always lies within the quantile. Moreover, we also see that the quantile at $t\approx0.5$ is decreasing until $t\approx 1.2$, while at the same time, the error between the mean estimate and the true IF is decreasing as well. Similarly, the quantile increases as the error increases around $t\in[1.2, 2.0]$. In contrast, the figures in the third and fourth rows and the third column show that the methods ``GHFS MLE on~\eqref{equ:old-ekf-ss}'' and KPT MLE give erroneous and overconfident quantiles. \looseness=-1

Table~\ref{tbl:learnt-params} shows the estimated model parameters of the GHFS MLE method from the same MC run plotted in Figure~\ref{fig:estimation}. We found that these estimated parameters are meaningful, especially for $\lambda$ and $b$ which determine the chirp prior's damping factor and stochastic volatility, respectively. When $\alpha$ is a constant, $\lambda$ and $b$ are estimated to have small values because the true chirp has no damping nor randomness. When $\alpha(t)=\expp^{-0.3 \,t}$, the estimated $\lambda\approx0.3$ is very close to the true damping value. When $\alpha$ is random, the estimated $b$ is no longer small, because it needs to account for the stochastic volatility of the randomised chirp. As for $V$'s parameters $\ell$ and $\sigma$, their estimated values do not change significantly with different forms of $\alpha$. This result is desired because the chirp amplitude $\alpha$ does not affect the chirp IF by definition. 

We next test if the proposed scheme also outperforms other methods when dealing with harmonic chirp signals. To do so, we set $J=3$ in~\eqref{equ:toymodel}, and all the other experiment settings remain unchanged. However, the GHFS method here is computationally demanding, since the Gauss--Hermite quadrature does not cope well with the increased state dimensionality (see, Section~\ref{sec:harmonic-chirp}), hence, we replace it by the cubature Kalman filter and smoother (CKFS). We compare our scheme to FastNLS, FHC, and KPT MLE which can straightforwardly tackle the harmonic chirp signal and are also the methods that stand out in Table~\ref{tbl:rmse-reuslts}. 

The RMSE results for the harmonic $J=3$ experiments are shown in Table~\ref{tbl:results-rmse-harmonic}. We see from the table that the proposed scheme still outperforms all the other methods. Moreover, there is a significant margin even when dealing with the constant $\alpha$; this is in contrast to the results for $J=1$ in Table~\ref{tbl:rmse-reuslts}. It is worth noting that although the FastNLS method has the best mean result for random $\alpha$, its median and minimum are not the best. 

In Figure~\ref{fig:estimation-harmonic}, we plot the fundamental IF estimates of one MC run in the harmonic IF estimation experiment. This figure shows that the CKFS MLE using the proposed model substantially outperforms the other methods, and that the method produces a reasonable quantification of the uncertainty quantile. On the other hand, the competing methods all start to diverge at around $t=1.4$~s. These methods diverge because we see from the top figure that the chirp amplitudes are relatively small for around $t\geq 1.4$~s due to the random perturbation.

Table~\ref{tbl:complexity} compares the computational complexities of the most important methods used in the experiment. From the table, we see that ANF and the state-space based methods (i.e., EKFS, GHFS, CKFS, and KPT) stand out, as their complexities are linear in the number of measurements $T$ which is usually the dominating factor in signal processing. On the other hand, the state-space based methods are quadratic or cubic in the number of harmonics $J$, since the state dimension increases as $J$ increases. Among the state-space based methods, the complexity of KPT is marginally better. The complexities of FastNLS and FHC depend on their grid-search resolutions and how their time-windows are selected. Based on the suggestions in~\cite{Jensen2017} and~\cite{Nielsen2017}, FastNLS and FHC are approximately sub-quadratic and sub-cubic, respectively, in the time-window length $\widehat{T}$, and this is to be multiplied with the number of time-windows $W$.

\begin{table}[h!]
	\centering
	\caption{Comparison of the computationally complexities. The notation $W$ and $\widehat{T}$ stand for the number of time windows and the window length, respectively. The notation $K$ and $F$ are the number of grid points for searching the chirp rate and the fundamental frequency, respectively. The approximations for FastNLS and FHC are obtained by using the recommended grid setting $F = O(\widehat{T} \, J)$ and $K = O(\widehat{T} \, J)$ as per~\cite{Jensen2017}. The methods below the dashed line use the proposed model, and recall that $d_v$ is the dimension of the prior for $V$. }
	\label{tbl:complexity}
	\begin{tabular}{@{}ll@{}}
		\toprule
		Method           & Computational complexity       \\ \midrule
		ANF~\cite{Niedzwiecki2011}              & $\mathcal{O}(J \, T)$                      \\
		EKFS on~\eqref{equ:old-ekf-ss}~\cite{Scala1996} &  $\mathcal{O}((2 \, J + d_v)^2 \, T)$                              \\
		GHFS on~\eqref{equ:old-ekf-ss}~\cite{Scala1996} &  $\mathcal{O}((2 \, J + d_v)^3 \, T)$                  \\
		FastNLS~\cite{Nielsen2017}          & $\begin{aligned}[t]
			&\mathcal{O}\bigl(W \, (F\log(F) + F \, L)\bigr)\\
			&\approx \mathcal{O}\bigl(W (J \, \widehat{T} \log(J \, \widehat{T}) + J \, \widehat{T} \, L)\bigr)
		\end{aligned}$\\
		FHC~\cite{Jensen2017}              & $\mathcal{O}(W \, K \, F \log(K \, F)) \approx \mathcal{O}\bigl(W \,J^2 \, \widehat{T}^2 \log(J \, \widehat{T})\bigr)$ \\
		KPT~\cite{Shi2017}              & $\mathcal{O}((J + 2)^2 \, T)$              \\
		\cdashlinelr{1-2}
		EKFS             & $\mathcal{O}((2 \, J + d_v)^2 \, T)$       \\
		GHFS/CKFS             & $\mathcal{O}((2 \, J + d_v)^3 \, T)$       \\\bottomrule
	\end{tabular}
\end{table}

\subsection{Gravitational wave frequency estimation}
\label{sec:gw}
We use the proposed model to estimate the frequency of a gravitational wave (GW) signal. The GW signal is taken from a well-known binary black hole collision event GW150914 in 2016~\cite{LIGO2016}. This signal is evenly sampled at $16,384$~Hz and contains 3,441 measurements. Details regarding this data are found in~\cite{LIGO2016}. To estimate the chirp amplitude and IF, we use the GHFS MLE method which has the best RMSE statistic in the synthetic experiment. Our method works out-of-the-box for this data, since the model parameters are automatically inferred from the data via MLE.

\begin{figure}[t!]
	\centering
	\includegraphics[width=.99\linewidth]{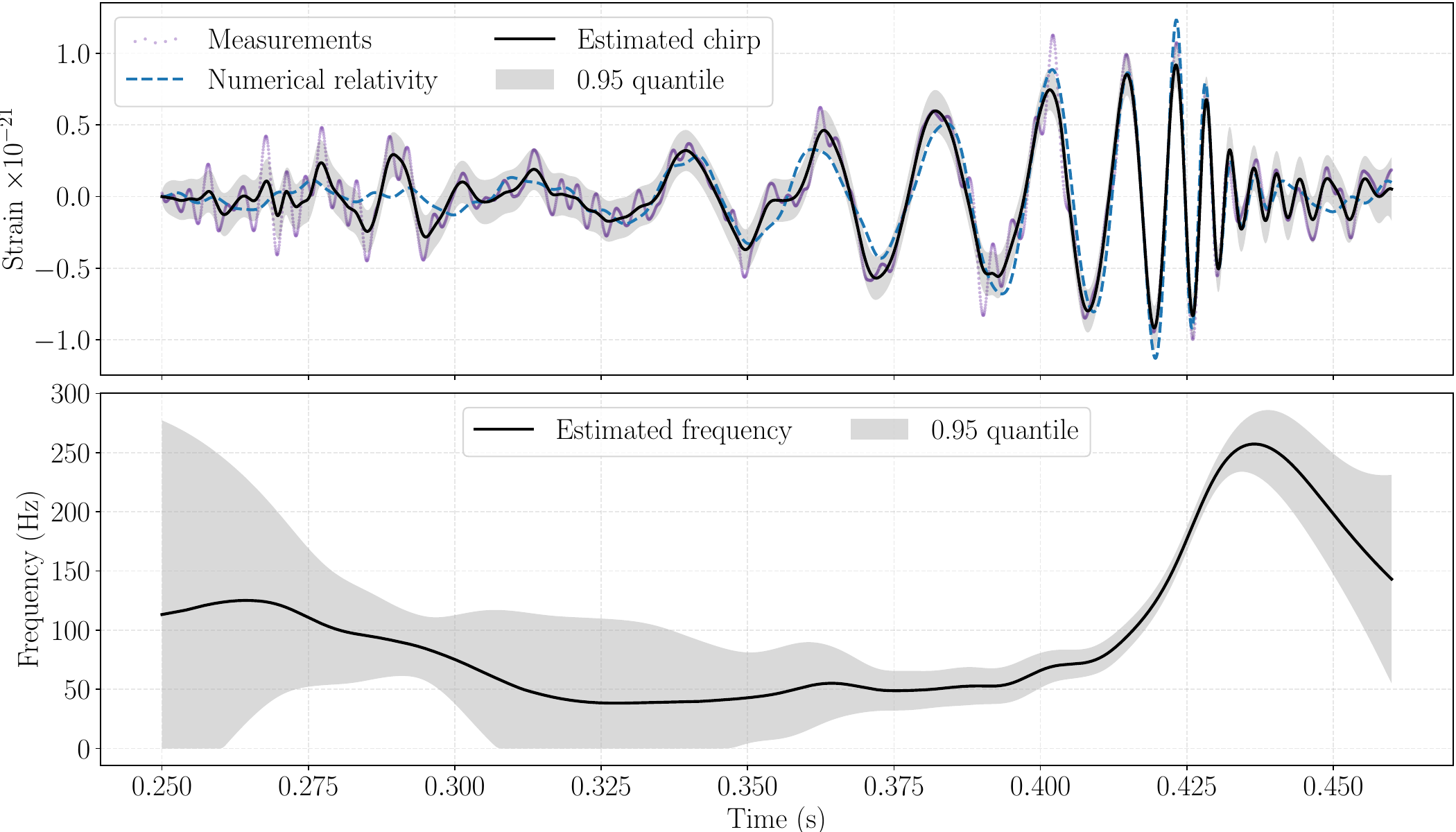}
	\caption{Gravitational wave frequency (and chirp) estimation by using the GHFS MLE method on the model in~\eqref{equ:sde-chirp}. Estimated model parameters are $\lambda=1.59\times 10^1$, $b=2.12$, $\delta=7.78\times 10^{-15}$, $\ell=4.21\times 10^{-2}$, $\sigma=1.30\times10^{2}$, and $m^V_0=3.91\times 10^{-4}$.}
	\label{fig:gw}
\end{figure}

\begin{figure}[t!]
	\centering
	\includegraphics[width=.99\linewidth]{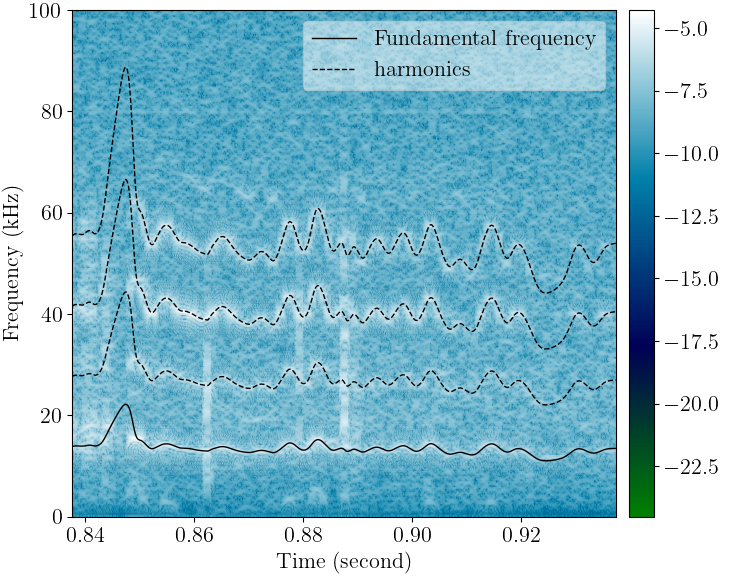}
	\includegraphics[width=.99\linewidth]{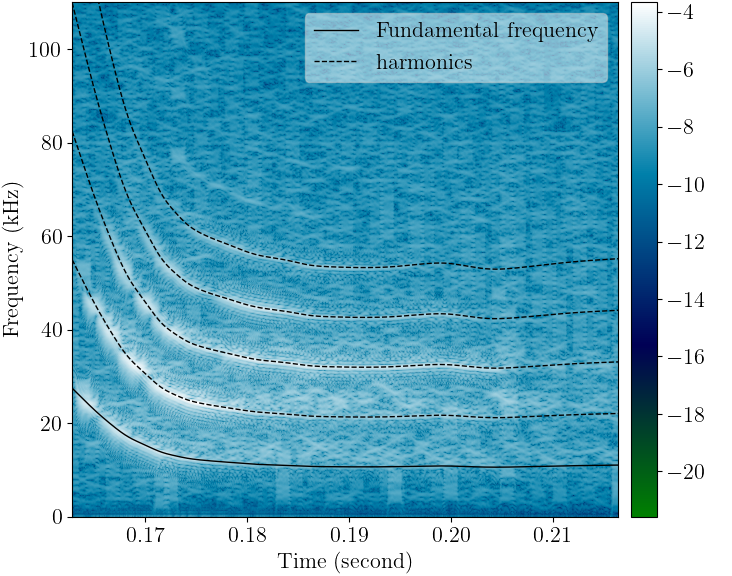}
	\caption{Estimation of the instantaneous fundamental and harmonic frequencies of two calls from \textit{Myotis myotis} (top) and \textit{Eptesicus nilssonii} (bottom). The unit of the colour bar is the $\log_{10}$ decibel. }
	\label{fig:bats}
\end{figure}

The GW chirp and frequency estimates are shown in Figure~\ref{fig:gw}. From this figure, we see that the frequency slowly increases from around $50$~Hz at $t\approx 0.35$~s to $80$~Hz at $t\approx 0.4$~s. Then, starting from $t\approx 0.4$~s, the frequency sharply increases to around $250$~Hz until $t\approx 4.3$~s, followed by a sharp decrease from $t\geq 4.3$. This estimation result makes sense, as it follows the theoretical prediction from the GW model of merging binary black holes~\cite{LIGO2016}. Moreover, the estimated frequency is close to the estimate provided in~\cite[][Fig. 1, left bottom]{LIGO2016}. The quantile indicates that the estimator is confident in the region $t\in[0.4, 0.43]$ which is the most interesting section. \looseness=-1

We would like to note that astronomers are typically more interested in visualising the spectrogram of the GW signal than tracking the IF. Thus, we do not claim that our method is a practical approach for profiling GW signals compared to what is currently used. 

\subsection{Analysing bat calls}
\label{sec:bat-calls}
To show that our method works for real-world chirp signals that have harmonic frequencies, we apply the method to the sounds of two European bat calls, \textit{Myotis myotis} and \textit{Eptesicus nilssonii}. These sounds are sampled at $250,000$~Hz, and more details of the data are found at \url{avisoft.com/batcalls}. We plot the spectrogram of the two sounds in Figure~\ref{fig:bats}, and we can clearly see that \textit{Myotis myotis} and \textit{Eptesicus nilssonii} have four and five dominant harmonics, respectively. To track these frequencies, we apply the CKFS MLE method which has the best overall performance shown in Table~\ref{tbl:results-rmse-harmonic}.

The results are shown in Figure~\ref{fig:bats}. We see that the estimated fundamental IFs accurately track the ground truth plotted in the spectrograms, for both sounds. The harmonic IFs computed by the estimated fundamental IFs are also close to the truth. It is worth remarking that estimating the fundamental IF for \textit{Myotis myotis} (top figure) is challenging, because the IF changes rapidly in time, and there are missing sections (e.g., around $t=0.845$~s) and distortions in the signal due to poor sound quality.\looseness=-1

In this experiment, the sound signal of \textit{Myotis myotis} has $25,335$ samples. On a standard personal computer (Intel Core i9-10900K, Ubuntu 20.04, and Python JAX implementation), the CKFS method with our model takes around 0.8 seconds. This shows that the proposed scheme is efficient enough to be used in practice.

\section{Conclusions}
\label{sec:conclusion}
In this paper, we have designed a hierarchical Gaussian process model for joint modelling of chirp signals and their instantaneous frequencies. This enables us to estimate a wide class of chirp signals and their instantaneous frequencies in continuous-time. In order to carry out the estimation and computation in practice, the model is represented in a non-linear stochastic differential equation, and the posterior distribution is computed using stochastic filtering and smoothing methods (e.g., sigma-points filters and smoothers). Moreover, the model parameters which control the characteristics of the chirp and IF, can be determined via maximum likelihood estimation and automatic differentiation. This makes the method an out-of-the-box approach for processing signals. To characterise the mean squared estimation error of the proposed scheme, we computed a Cram\'{e}r--Rao lower bound and a theoretical upper bound for it. Lastly, the experimental results show that the proposed model significantly outperforms a number of state-of-the-art methods, and that the model is applicable to real data (e.g., gravitational wave and bat calls) as well.

\section*{Acknowledgement}
The authors' contributions are given as follows. Zheng Zhao came up with the idea, did all the theoretical analysis and the experiments, and wrote the initial manuscript. Simo S\"{a}rkk\"{a} gave useful comments and helped revising the paper. Jens Sj\"{o}lund and Thomas Sch\"{o}n shaped the narrative of the paper and contributed in writing the final manuscript. 

\appendices

\section{~}
\label{append:matern32}
The explicit discretisation of the Mat\'{e}rn $3 \, / \, 2$ SDE in~\eqref{equ:disc-lcd} is given by
\begin{equation*}
	\expp^{\Delta_k \, M} = \expp^{-\Delta_k \, \gamma}
	\begin{bmatrix}
		1 + \Delta_k \, \gamma & \Delta_k               \\
		-\Delta_k \, \gamma^2  & 1 - \Delta_k \, \gamma
	\end{bmatrix},
\end{equation*}
\begin{equation*}
	\begin{split}
		&\Lambda(\Delta_k) \coloneqq \int^{\Delta_k}_0 \expp^{(\Delta_k - s) \, M} \, N \, N^\trans \, \big(\expp^{(\Delta_k - s) \, M}\big)^\trans \diff s \\
		&=
		\begin{bmatrix}
			\sigma^2 - \beta \, (2 \, \eta + 2 \, \eta^2 + 1) & 2 \, \Delta_k^2 \, \gamma^3 \, \beta                            \\
			2 \, \Delta_k^2 \, \gamma^3 \, \beta              & \gamma^2 \, (\sigma^2 + \beta \, (2 \, \eta - 2 \, \eta^2 - 1))
		\end{bmatrix},
	\end{split}
\end{equation*}
where $\gamma\coloneqq \sqrt{3} \, / \, \ell$, $\beta \coloneqq \sigma^2 \exp(-2 \, \eta)$, and $\eta \coloneqq \Delta_k \, \gamma$.

\section{~}
\label{append:baseline}
Parameters and settings of the baseline methods in Section~\ref{sec:experiment-toy} are given as follows.
\begin{itemize}
	\item The Hilbert transform method uses a cascaded second-order forward-backward digital filter to preprocecss the noisy measurements. The filter uses an 8-th order Butterworth design with a critical frequency of $18$~Hz.
	\item The spectrogram method uses the same filtering procedure as in the Hilbert transform method to preprocess measurements. We choose the cosine time window of length 450 and 449 overlaps.
	\item The (pilot) adaptive notch filter is from~\cite[][Table II]{Niedzwiecki2011}. Its parameters ($\mu=0.015$, $\gamma_w = \mu^2 / 2$, and $\gamma_\alpha=\mu \, \gamma_w / 4$ using the notation in~\cite{Niedzwiecki2011}) are selected according to the guidance provided in~\cite[][pp. 2032]{Niedzwiecki2011}.
	\item The polynomial method fits the IF by a 10-th order polynomial, the coefficients of which are determined via maximum likelihood estimation and implemented using Levenberg--Marquardt optimisation.
	\item The FastNLS and FHC methods use a window length of 300 samples and 298 overlaps.
\end{itemize}

\def\bibfont{\small}

\bibliographystyle{IEEEtran}
\bibliography{refs}

\begin{IEEEbiography}[{\includegraphics[width=1in,height=1.25in,clip,keepaspectratio]{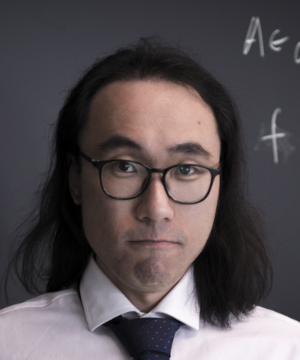}}]{Zheng Zhao}
	received his D.Sc. degree from Aalto University, Espoo, Finland in 2021. He is currently working as a WASP postdoctoral researcher in the Department of Information Technology, Uppsala University, Sweden. His research interests include stochastic filtering, stochastic differential equations, Gaussian processes, signal processing, and machine learning.
\end{IEEEbiography}

\begin{IEEEbiography}[{\includegraphics[width=1in,height=1.25in,clip,keepaspectratio]{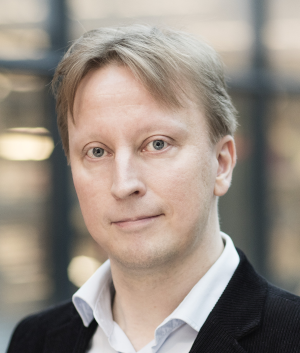}}]{Simo S\"{a}rkk\"{a}}
	received the M.Sc. (Tech.) and D.Sc. (Tech.) degrees from the Helsinki University of Technology, Espoo, Finland, in 2000 and 2006, respectively. He is currently an Associate Professor with Aalto University, Espoo, and an Adjunct Professor with the Tampere University and the LUT University. He is also affiliated with the Finnish Center of Artificial Intelligence (FCAI). His research interests include machine learning and multisensor data processing systems with applications in health and medical technology, location sensing, and inverse problems. 
\end{IEEEbiography}

\begin{IEEEbiography}[{\includegraphics[width=1in,height=1.25in,clip,keepaspectratio]{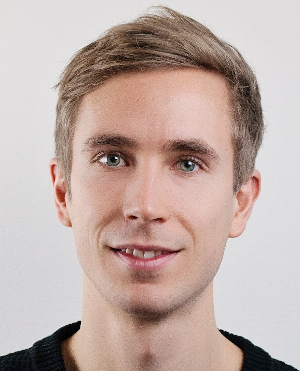}}]{Jens Sj\"{o}lund}
	received the B.Sc. and M.Sc. degree in engineering physics from the Royal Institute of Technology (KTH), Stockholm, Sweden, in 2010 and 2012, respectively, and the Ph.D. degree in biomedical engineering from Link\"{o}ping University, Link\"{o}ping, Sweden, in 2018. He is currently a WASP assistant professor of artificial intelligence with the Department of Information Technology, Uppsala University, Uppsala, Sweden. Prior to joining Uppsala University in 2021, he was a senior research scientist at Elekta, where he established a track record as a prolific inventor with more than 25 patent applications.
\end{IEEEbiography}

\begin{IEEEbiography}[{\includegraphics[width=1in,height=1.25in,clip,keepaspectratio]{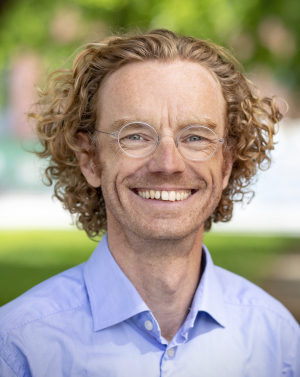}}]{Thomas Sch\"{o}n}
	received the B.Sc. degree in business administration and economics, the M.Sc. degree in applied physics and electrical engineering, and the Ph.D. degree in automatic control from Link\"{o}ping University, Link\"{o}ping, Sweden, in January 2001, September 2001, and February 2006, respectively. He is currently the Beijer Professor of artificial intelligence with the Department of Information Technology, Uppsala University, Uppsala, Sweden. In 2018, he was elected to The Royal Swedish Academy of Engineering Sciences (IVA) and The Royal Society of Sciences, Uppsala, Sweden. He was the recipient of the Tage Erlander Prize for natural sciences and technology in 2017 and the Arnberg Prize in 2016, both awarded by the Royal Swedish Academy of Sciences (KVA).
\end{IEEEbiography}

\end{document}